\documentclass[10pt]{article}
\usepackage[utf8]{inputenc} % allow utf-8 input
\usepackage{amsmath}
\usepackage{amsfonts}
\usepackage{amssymb}
\usepackage{mathrsfs}                  
\usepackage{amsthm}
\usepackage{enumitem}
\usepackage{amscd}
\usepackage{xcolor}
\usepackage[hyperfootnotes=false]{hyperref}
\usepackage{mathtools}
\usepackage{geometry}
\usepackage{bbm}
\usepackage{bm}
\usepackage{subcaption}
\usepackage{chemarrow}
\usepackage{graphicx}       % graphics
\usepackage{multirow}
\usepackage{latexsym,graphics,amscd} 
\usepackage{bigints}
\usepackage{booktabs} 
\usepackage[ruled,vlined]{algorithm2e}
\usepackage{tabularx} 
\usepackage{threeparttable}
\usepackage{lscape}
\usepackage{caption} 
\usepackage{array} 
\usepackage{makecell}
\usepackage{url}
\usepackage[final]{pdfpages}
\usepackage{hyperref}
\hypersetup{
	colorlinks=true,
	citecolor=blue,
	linkcolor=blue,
	filecolor=magenta,      
	urlcolor=blue,
}
\usepackage{fancyhdr}       % header

\usepackage{natbib}
\numberwithin{equation}{section}

\allowdisplaybreaks

\theoremstyle{plain}
\newtheorem{theorem}{Theorem}[section]
\newtheorem{proposition}[theorem]{Proposition}
\newtheorem{lemma}[theorem]{Lemma}

\theoremstyle{definition}
\newtheorem{definition}[theorem]{Definition}
\newtheorem{example}[theorem]{Example}

\theoremstyle{remark}

\numberwithin{equation}{section}
\addtocounter{footnote}{1}
\makeatletter
\@addtoreset{table}{bsection}
\def\thetable{\thesection.\@arabic\c@table}
\def\fps@table{h, t}
\@addtoreset{equation}{section}

\makeatother

\providecommand{\norm}[1]{\lVert#1\rVert}
\newcommand{\bfi}{\bfseries\itshape}
\newcommand{\vertiii}[1]{{\left\vert\kern-0.25ex\left\vert\kern-0.25ex\left\vert #1 
    \right\vert\kern-0.25ex\right\vert\kern-0.25ex\right\vert}}
\newsavebox{\savepar}

\makeatletter
\makeatletter

\usepackage{scalerel,stackengine}
\stackMath
\newcommand\reallywidehat[1]{%
\savestack{\tmpbox}{\stretchto{%
  \scaleto{%
    \scalerel*[\widthof{\ensuremath{#1}}]{\kern-.6pt\bigwedge\kern-.6pt}%
    {\rule[-\textheight/2]{1ex}{\textheight}}%WIDTH-LIMITED BIG WEDGE
  }{\textheight}% 
}{0.5ex}}%
\stackon[1pt]{#1}{\tmpbox}%
}

\newcommand{\E}{\mathbb{E}}

\newcommand{\diag}{{\text{diag}}}

\newcommand{\R}{\mathbb{R}}

\newcommand{\Z}{\mathbb{Z}}
\newcommand{\N}{\mathbb{N}}
\newcommand{\M}{\mathbb{M}}

\begin{document}
\title{\textbf{%
    Memory of recurrent networks:\\ Do we compute it right?
}}
\author{Giovanni Ballarin$^{1}$, Lyudmila Grigoryeva$^{2,3}$, Juan-Pablo Ortega$^{4}$}
\date{\today}
\maketitle

\begin{abstract}
    Numerical evaluations of the memory capacity (MC) of recurrent neural networks reported in the literature often contradict well-established theoretical bounds. In this paper, we study the case of linear echo state networks, for which the total memory capacity has been proven to be equal to the rank of the corresponding Kalman controllability matrix. We shed light on various reasons for the inaccurate numerical estimations of the memory, and we show that these issues, often overlooked in the recent literature, are of an exclusively numerical nature. More explicitly, we prove that when the Krylov structure of the linear MC is ignored, a gap between the theoretical MC and its empirical counterpart is introduced. As a solution, we develop robust numerical approaches by exploiting a result of MC neutrality with respect to the input mask matrix. Simulations show that the memory curves that are recovered using the proposed methods fully agree with the theory.
\end{abstract}

\bigskip


%\textbf{JEL:} %C53, C45, E17

\makeatletter
\addtocounter{footnote}{1} \footnotetext{ 
Department of Economics, University of Mannheim, L7, 3-5, Mannheim, 68131, Germany. {\texttt{Giovanni.Ballarin@gess.uni-mannheim.de} }}
\addtocounter{footnote}{1} \footnotetext{%
Faculty of Mathematics and Statistics, University of St. Gallen, Bodanstrasse 6, CH-9000 St.~Gallen, Switzerland. {\texttt{Lyudmila.Grigoryeva@unisg.ch} }}
\addtocounter{footnote}{1} \footnotetext{%
Honorary Associate Professor, Department of Statistics, University of Warwick, Coventry CV4 7AL, UK. {\texttt{Lyudmila.Grigoryeva@warwick.ac.uk} }} 
\addtocounter{footnote}{1} \footnotetext{%
Division of Mathematical Sciences, School of Physical and Mathematical Sciences,
Nanyang Technological University,
21 Nanyang Link,
Singapore 637371.
{\texttt{Juan-Pablo.Ortega@ntu.edu.sg}}}
\makeatother

\newpage
\setcounter{footnote}{0} 

\tableofcontents
\newpage

\section{Introduction}
\label{Introduction}
Recurrent Neural Networks (RNNs) are among the most widely used machine learning tools for sequential data processing \citep{Sutskever2014}. Despite the rising popularity of transformer deep neural architectures \citep{Vaswani2017, galimberti2022designing, anastasis:hypertransformer}, in particular, in natural language processing, RNNs remain more suitable in a significant range of real-time and online learning tasks that require handling one element of the sequence at a time. The key difference is that transformers are designed to process entire time sequences at once, using self-attention mechanisms to focus on particular entries of the input, while RNNs use hidden state spaces to retain a memory of previous elements in the input sequence, which makes memory one of the most important features of RNNs. Multiple attempts have been made in recent years to design quantitative measures and characterize memory in neural networks in general~\citep{Vershynin2020,Koyuncu2023} and their recurrent versions, in particular, \citep{Haviv2019,Li2021}. 

The notion of {\bfi memory capacity} (MC) in recurrent neural networks was first introduced in~\cite{Jaeger:2002}, with a particular focus on the so-called echo state networks (ESNs) \citep{Matthews:thesis, Matthews1994, Jaeger04}, which are a popular family of RNNs within the reservoir computing (RC) strand of the literature that have shown to be universal approximants in various contexts \citep{RC7, RC8, RC20}. RC models are state-space systems whose state map parameters are randomly generated and which can be seen as RNNs with random inner neuron connection weights and a readout layer that is trained depending on the learning task of interest. Memory capacity has been proposed as a measure of the amount of information stored in the states of a state-space system in relation to past inputs. It has been commonly accepted as a valuable metric to evaluate the network's ability to store and extract important information from processed input signals over time. Extensive work has been done in the reservoir computing literature both in the setting of linear~\citep{Hermans2010, dambre2012, esn2014, linearESN, Goudarzi2016, Xue2017}, echo state shallow \citep{White2004,farkas:bosak:2016, Verzelli2019a}, and deep architectures~\citep{bruges:deepRC,gallicchioShorttermMemoryDeep2018}. Memory capacity definitions exploited extensively in the literature are based on a natural observation that the ability of the network to memorize previous inputs can be quantitatively assessed by the correlation between the outputs of the network and its past inputs. Originally, independent and identically distributed input sequences were used for these measurements and only some recent references discuss the case of temporary dependent inputs (for example, \citealt{dambre2012,charles2014short,RC4pv,Charles2017,marzen:capacity,RC15}).  Proposals of other memory measures have also been discussed in the literature, with Fischer information-based criteria \citep{Ganguli2008, Tino2013, Livi2016, tino:symmetric} among those.

Over the past few years, a series of papers presented analytical expressions for the capacity of time-delay reservoirs \citep{GHLO2014_capacity, RC3}. The main interest of memory measures, in general, and capacities, in particular, is related to their use in architecture design. Once the memory capacity expression as a function of the network (hyper-)parameters is available, one could use it in order to design memory-optimal network architectures. This seemed to be especially important for nontrainable random connectivity neural networks, where the choice of sampling and network structure can be informed by maximizing network capacities. This direction was pursued in numerous studies, with many of those focusing on linear and echo state networks \citep{pesquera2012, GHLO2014, ortin2019tackling, ortin2020delay}.

In this paper, we place ourselves in the setting of linear recurrent neural networks. A recent contribution in the literature in this framework is \cite{RC15}. Interestingly, it is proved in this reference that linear systems with white noise inputs (not necessarily independent) and non-singular state autocovariance matrices {\it automatically have full memory capacity}, which coincides with the dimension of the state space or, equivalently, the number of neurons in the hidden layer. Moreover, while  \cite{Jaeger:2002} shows that the memory capacity is maximal if and only if Kalman's controllability rank condition \citep{Kalman2010, sontag1991kalman, sontag:book} is satisfied, \cite{RC15} also proves that the memory capacity of linear systems is given {\it exactly by the rank of the Kalman controllability matrix}. These results {\bfi contradict numerous studies} in the literature that report empirical evaluations of the memory capacity of linear recurrent networks inconsistent with the result in \cite{RC15}. We shall use the term {\bfi linear memory gap} to denote the difference between empirically measured memory capacities of linear networks and their provable theoretical values. To the best of our knowledge, this paper is the first to shed light on the nature of this incoherence. We argue that the memory gap originates from pure numerical artifacts overlooked by many previous studies and propose robust techniques that allow for accurate estimation of the memory capacity, which renders full memory results for linear RNNs in agreement with the well-known theoretical results. We claim that multiple efforts in the literature to optimize the memory capacity of linear recurrent networks are hence afflicted by numerical pathologies and convey misleading results.

Specific numerical issues that arise at the time of memory computation and which, as we explain later, are attributed to the ill-conditioning of Krylov matrices, were noticed by some authors in empirical experiments. However, no rigorous explanation has been found so far. Instead, the literature has been developing in the following two directions: the first one designs specific network architectures that are not susceptive to these phenomena, and the second one tunes the hyperparameters to achieve empirical capacity maximization for a given network architecture.

The first research strand finds configurations for which the memory gap is absent or minimal. For example, for ESNs with nonlinear hyperbolic tangent activation, based on empirical insights, \cite{farkas:bosak:2016} proposes an orthogonalization process that improves memory capacity evaluation in simulations (similar ideas in the vein of orthogonal neural networks are also developed in \citealt{White2004}). \cite{ogESN2012} provides designs for ESN reservoir matrices, called RingOfNeurons and ChainOfNeurons, that are inspired by rotation matrices and are based on the memory capacity maximization idea. Full memory of the delay-line and cyclic reservoirs has also been reported in \cite{Rodan2011, rodanSimpleDeterministicallyConstructed2012}. \cite{Tino2013} contribute to the same direction characterizing the MC for a particular type of reservoir connectivity, namely symmetric reservoir matrices. In this paper, we rigorously show how and why particular choices of connectivity matrices in the linear setting surpass the ill-conditioning problem and hence exhibit no memory gap by construction.

The second strand of the literature focuses on the question of whether some hyperparameter choices may maximize the memory capacity of the network. We find that many of these contributions propose explanations of the empirical findings that, in the light of \cite{RC15}, are not always entirely correct. In particular, they focus on hyperparameters or sampling distributions of the state map matrix parameters (within the family of regular laws) that have provably no effect on the memory capacity. For example, \cite{gallicchioSparsityReservoirComputing2020} makes a numerical argument for the sparsity in ESN reservoir matrices, since it claims that it maximizes memory and the ``effective dimension'' of the state space. This claim is based on a numerical artifact that is mainly due to the different spectral properties of random matrix ensembles of different sparsity degrees. Another example is \cite{Aceituno2017} which studies the average eigenvalue modulus of the reservoir matrix as a proxy for memory and suggests, in particular, using circulant matrices to maximize memory. 

We conclude this brief literature review by mentioning a few references, which, in our view, are the closest to obtaining a satisfactory explanation of the memory gap phenomenon.  \cite{whiteakerMemoryEchoState2022} correctly identifies the importance of the controllability matrix rank, even though it considers a nonlinear setting. Some intuitive links between the rank of the controllability matrix and memory capacity are discussed in \cite{verzelliInputtoStateRepresentationLinear2021} and \cite{Whiteaker2022}. MC is studied in that paper through simulations that yield an incorrect conclusion as to the imperfect memory in some linear ESN (LESN) designs. Finally, \cite{Hermans2010} studies memory in the case of continuous-time models and contains a version of the result of input mask neutrality that we present later in the paper.

This paper contains two main contributions. First, we address the methods of empirical memory estimation commonly exploited in the literature. In particular, while \cite{RC15} together with \cite{RC21} prove that $N$-dimensional linear state-space or RNN systems with randomly sampled matrix parameters of the state map have almost surely full memory of $N$, both Monte Carlo simulation and algebraic techniques, which we call {\it na\"ive}, exhibit numerical issues and lead either to over- or underestimation of the memory capacity in this setting. Unfortunately, these approaches were followed in many studies and led to the devising of recommendations for optimal random architectures based on numerical pathologies, which we discussed in previous paragraphs. Second, the insight into numerical issues led us to develop numerically robust algorithms for memory capacity evaluations. One of the results that we derive is the so-called neutrality of memory capacity to input mask, which we build upon in order to propose a numerically stable memory capacity empirical evaluation scheme using subspace methods. We call these newly introduced techniques the {\bfi orthogonalized subspace} and {\bfi averaged orthogonalized subspace} methods. We hope with our proposal to give closure to a long line of contributions in the literature that attempts to maximize capacities that are almost surely and provably full, to begin with.

The paper is structured as follows. In Section~\ref{section:naive} we present memory capacity estimation approaches as currently used in the literature of reservoir computing and, specifically, ESN models. We highlight the main issues that arise with these methods, which have lead to significant efforts to seemingly maximize memory properties of linear echo state networks. In Section~\ref{section:subspace} we present our new method based on linear subspaces induced by the Krylov structure of the controllability matrix. This approach recovers the theoretical memory properties of LESNs and is immediate to implement; we also proposed an improved version that relies on a novel result of the invariance of memory capacity with respect to the choice of the input mask. We conclude with Section~\ref{section:conclusion}.

\subsection{Code} 

All codes necessary to reproduce numerical results presented in the paper are publicly available at \url{https://github.com/Learning-of-Dynamic-Processes/memorycapacity}.

\subsection{Notation} 

Column vectors are denoted by bold lowercase symbols like $\mathbf{r}$. Given a vector $\mathbf{v} \in \mathbb{K}  ^n $, we denote its entries by $v_i$, with $i \in \left\{ 1, \dots, n
\right\} $.  
We denote by $\mathbb{M}_{n ,  m }$ the space of $\mathbb{K}$-valued $n\times m$ matrices with $m, n \in \mathbb{N} $. The choice of $\mathbb{K}$ is either $\mathbb{C}$ or $\mathbb{R}$, which will be clear from the context. When $n=m$, we use the symbol $\mathbb{M}  _n $ to refer to the space of square matrices of order $n$. Given a vector $\mathbf{v} \in {\mathbb{K}}  ^n $, we denote by ${\rm diag} (\mathbf{v})$ the diagonal matrix in $\mathbb{M}  _n $ with the elements of $\mathbf{v}  $  as diagonal entries. Given a matrix $A \in \mathbb{M}  _{n , m} $, we denote its components by $A _{ij} $ and we write $A=(A_{ij})$, with $i \in \left\{ 1, \dots, n\right\} $, $j \in \left\{ 1, \dots m\right\} $. Given a vector $\mathbf{v} \in \mathbb{R}  ^n $, the symbol $\| \mathbf{v}\|  $ stands for its Euclidean norm. For any $A \in \mathbb{M}  _{n , m} $,  $\|A\|   $ denotes its matrix norm induced by the Euclidean norms in $\mathbb{K}^m $ and $\mathbb{K} ^n $,  and satisfies that $\|A\| =\sigma_{{\rm max}}(A)$, with $\sigma_{{\rm max}}(A)$  the largest singular value of $A$  \citep{horn:matrix:analysis}. Whenever $\mathbb{K}=\mathbb{C}$, for  $A\in \mathbb{M}  _{n , m}$, we denote by $A^\ast \in \mathbb{M}  _{m , n}$ its conjugate transpose defined by
$
\left({A}^{\ast}\right)_{i j}=\overline{{A}_{j i}},
$
where the bar denotes the complex conjugate. For $A\in \mathbb{M}  _{n , m}$, $A^\top$ denotes its transpose, while $\mathcal{C}(A)\subset \mathbb{K}^n$ and $\mathcal{C}(A^\top)\subset \mathbb{K}^m$ are its column and row spaces, respectively.

\section{Linear Memory Capacity}
\label{section:naive}

Consider the linear echo state network (LESN) defined by the following two equations:
\begin{align}
	\mathbf{x}_t & = {A} \mathbf{x}_{t-1} + {C} \mathbf{z}_t + \boldsymbol{\zeta} \label{eq:ESN_def_1} , \\
	\mathbf{y}_{t} & = {W}^\top \mathbf{x}_t , \label{eq:ESN_def_2}
\end{align}
for $t \in \Z_-$, where $ \mathbf{z} \in (\R^d)^{\Z_-}$ are the inputs, $ \mathbf{x} \in (\R^N)^{\Z_-}$ are the states, and $ \mathbf{y} \in (\R^m)^{\Z_-}$ are the outputs,  $d, m, N \in \N$. The states in \eqref{eq:ESN_def_1} are defined using the {\bfi  reservoir (connectivity) matrix} $A\in \mathbb{M}_{N}$, the {\bfi  input mask} $C\in \mathbb{M}_{N, d}$, and the {\bfi input shift} $\bm{\zeta} \in \R^N$, and are mapped to the outputs via the affine readout map with associated {\bfi readout weights matrix} ${W}\in \mathbb{M}_{N,m}$ which can be adjusted to incorporate the intercept term. In the rest of the paper, we consider one-dimensional inputs and outputs and hence use bold symbols $\mathbf{C}, \mathbf{W} \in {\Bbb R}^N$ to denote the input mask and the readouts vectors, respectively. 

We shall focus on state-space systems of the type \eqref{eq:ESN_def_1}-\eqref{eq:ESN_def_2} that determine an {\bfi input/output} system. This happens in the presence of the so-called {\bfi  echo state property (ESP)}, that is, when for any  $ \boldsymbol{z} \in \R^{\Z_-}$ there exists a unique $ \mathbf{y} \in \R^{\Z_-}$ such that \eqref{eq:ESN_def_1}-\eqref{eq:ESN_def_2} hold. One can require that the ESP holds only on the level of the state equation, that is that for any  $ \mathbf{z} \in (\R^d)^{\Z_-}$ there exists a unique $ \mathbf{x} \in (\R^N)^{\Z_-}$ such that \eqref{eq:ESN_def_1} holds. In Proposition 4.2 in \cite{RC16} it is proved  that the state equation associated to \eqref{eq:ESN_def_1} has a unique state-solution  $ \mathbf{x} \in \ell_-^\infty(\mathbb{R}^N)$ for each input in ${\bf z} \in \ell_-^\infty(\mathbb{R})$ (we call this property the $(\ell_-^\infty(\mathbb{R}^N), \ell_-^\infty (\mathbb{R}))$-ESP) if and only if the spectral radius of $A$ is strictly smaller than $1$, that is $\rho(A)<1$. We recall that the inputs $\mathbf{z}\in\ell_-^\infty (\mathbb{R})$ and the inputs $\mathbf{x}\in\ell_-^\infty(\mathbb{R}^N)$ are the left-infinite $\mathbb{R}^N$- and $\mathbb{R}$-valued sequences, respectively, with finite supremum norm $\|\cdot\|_{\infty}$, that is $\|\mathbf{z}\|_{\infty} := \sup_{t\in \Z_-}\{|z_t|\}<\infty$ and $\|\mathbf{x}\|_{\infty} := \sup_{t\in \Z_-}\{\|\mathbf{x}_t\|\}<\infty$ with $\|\cdot\|$ the Euclidean norm. Under the hypothesis $\rho(A)<1$, the unique solution $ \mathbf{x} \in \ell_-^\infty(\mathbb{R}^N)$ of \eqref{eq:ESN_def_1} associated to the input $ \mathbf{z} \in \R^{\Z_-}$ is given by the series
\begin{equation}
	\label{solution state LESN}
	\mathbf{x} _t=\sum_{j=0} ^{\infty}A ^j\mathbf{C} {z}_{t-j}, \quad \mbox{$t \in \mathbb{Z}_{-}$.} 
\end{equation}

In this paper, we consider inputs that are realizations of variance-stationary discrete-time stochastic $\mathbb{R}$-valued processes  $ \mathbf{z}=\left( z_t\right)_{t\in \Z_-}$. Additionally, since we study only the memory reconstruction information processing tasks, the target process $\mathbf{y}$ is a forward-shifted version of the input process $\mathbf{z}$. In the stochastic setting, one can show that the same condition $\rho(A)<1$ is sufficient for the almost sure unique existence of a solution of \eqref{eq:ESN_def_1}. More explicitly, if $\rho(A)<1$ and the input process ${\bf z} $ is such that $\operatorname{Var}(z _t)< c $ for all $t \in \mathbb{Z}_{-} $ and a finite constant $c >0 $, then there exists an a.s. unique sequence of random variables $\mathbf{x} $ such that 
\begin{equation}
	\label{limit for finite variance}
	\sum_{j=0} ^{T}A ^j\mathbf{C} {z}_{t-j}\xrightarrow[T \rightarrow \infty]{ L^2} \mathbf{x} _t, \quad \mbox{$t \in \Bbb Z_- $}.
\end{equation}
This statement is a corollary of \cite{luetkepohl:book}, Proposition C.9, which requires the absolute summability of the sequence $\left\{A ^j \mathbf{C}\right\}_{j \in \mathbb{N}} $ which is, in turn, a consequence of the hypothesis $\rho(A)<1$ and part (i) of Proposition 4.2 in \cite{RC16}. Additionally, a proof similar to the one of Proposition 4.1 in \cite{RC15} guarantees that  if ${\bf z}  $ is variance stationary, then so is $\mathbf{x}  $. Statements of this type in which the $L ^2 $ convergence is replaced by metric convergence in the Wasserstein space can be found in \cite{RC27}.

In contrast to conventional recurrent neural networks, where all the network weights (parameters) are subject to training,  the parameters of the state equations of reservoir systems are \textit{fixed}, and exclusively the readout map is estimated based on the learning task of interest. More explicitly, within the reservoir computing paradigm, in the case of LESN, the matrix parameters ${A}$, $\mathbf{C}$ and $\boldsymbol{\zeta}$ are sampled randomly from (matrix) probability distributions prescribed a priori and $\mathbf{W}$ is estimated. The choice of the law and the properties of these parameters are known to have a significant impact on the performance of the ESN in practical applications. 

\subsection{Memory Capacity} 

The notion of memory capacities (MCs) has been introduced in~\cite{Jaeger:2002} in the context of recurrent neural networks and echo state networks \citep{Matthews:thesis, Matthews1994, Jaeger04}  as a way to measure the amount of information contained in the states of a  state-space system about the past inputs and to characterize the ability of the network to extract the dynamic features of processed signals. Following \cite{RC15}, given a variance-stationary input stochastic process $ \mathbf{z}=\left( z_t\right)_{t\in \Z_-}$, a state map that satisfies the ESP, and the associated variance-stationary state process $ \mathbf{x}=\left( \mathbf{x}_t\right)_{t\in \Z_-}$, $\mathbf{x}\in \mathbb{R}^N$, the $\tau$-lag memory capacity  of the state-space system with respect to $\mathbf{z}$, with $\tau \in \mathbb{N}$, is defined as
\begin{equation}\label{eq:MC_definition}
	\text{MC}_\tau := 
	1 - \frac{1}{\text{Var}(z_t)} 
	\min_{\mathbf{W} \in \mathbb{R}^{N}} 
	\mathbb{E} 
	\left[ \,
	\left({{z}_{t-\tau} - \mathbf{W}^\top \mathbf{x}_t} \right)^2
	\right],
\end{equation}
where we will often use that $\text{Var}(z_t) = \gamma(0)$ with $\gamma: \mathbb{Z}\mapsto \mathbb{R}$ being the autocovariance function of $\mathbf{z}$.

The {\bfi  total memory capacity} of an ESN is then given by the sum of the capacities at all lags, that is,
\begin{equation}
	\text{MC} := \sum_{\tau = 0}^\infty \text{MC}_\tau .
\end{equation}
It is important to underline that in contrast to what is sometimes defined in the literature (for example, \citealt{Rodan2011}), our definition of $\text{MC}$ {includes} lag 0. We are interested in the complete history of the process $z_t$ embedded in the states $\mathbf{x}_t$, including the present. This is consistent with the fact that a LESN where $A = \mathbb{O}_N$ has no memory of inputs at lags $\tau > 0$, but if $N = 1$, it still retains maximal memory as long as $\mathbf{C} \not= \mathbf{0}$. By definition, $\textnormal{MC}_\tau$ measures how much of the variance of input $z_{t- \tau}$ can be linearly reconstructed from the states $\mathbf{x}_t$. The higher $\textnormal{MC}_\tau$ is for large $\tau$, the longer the states contain the past history of a sequence of inputs. 

Under the assumption that ${\Gamma}_{\mathbf{x}} := \text{Var}(\mathbf{x}_t)$ is non-singular, $\textnormal{MC}_\tau$ has the closed-form expression (see Lemma 3.2, \citealt{RC15})
\begin{equation}\label{eq:MC_tau_covvar}
	\text{MC}_\tau =
	\frac{\text{Cov}(z_{t-\tau}, \mathbf{x}_t) {\Gamma}_{\mathbf{x}}^{-1} \text{Cov}(\mathbf{x}_t,z_{t-\tau})}
	{\text{Var}(z_t)}, \enspace \tau\in \mathbb{N}.
\end{equation}

\begin{example}[{\bf Delay reservoir}] \normalfont
	Consider the delay (or Takens) reservoir given by the reservoir matrix whose only non-zero elements are $A_{ij}=1$ for all $j\in \{1,\ldots, N-1\}$, $i=j+1$ (this is usually called a shift matrix), the input mask $\mathbf{C}$ whose all elements are zero except for the first one which is set to one, and the zero input shift $\boldsymbol{\zeta}$. In this case, for any $t\geq N$
	\begin{equation*}
		\boldsymbol{x}_t = \left( \begin{array}{c}
			z_{t} \\
			\vdots \\
			z_{t-N}
		\end{array}\right)
	\end{equation*}
	and $\textnormal{MC}_\tau = 1$ for $\tau \in \{0, \ldots, N\}$ while $\textnormal{MC}_\tau = 0$ for all $\tau \geq N + 1$. %Then, if the LESN task is one-step-ahead forecasting, $y_{t+1} = z_{t+1}$, the resulting echo state network is equivalent to a linear AR($N$) model (\cite{BrocDavisYellowBook, MR1278033}). We make this connection more precise below when discussing cyclic reservoirs.
\end{example}

\subsubsection{Fischer Memory}

An alternative concept developed in the literature that pertains to the memory features of recursive neural networks is that of the {\bfi  Fischer memory curve} (FMC). The idea is introduced in \cite{Ganguli2008} and consists in quantifying the impact of small variations on the current state $\boldsymbol{x}_t$. More precisely, assume that in \eqref{eq:ESN_def_1} the states are perturbed by i.i.d. noise $(\boldsymbol{\epsilon}_t)_{t \in\mathbb{Z}_-}$ and that $\boldsymbol{\zeta} = \mathbf{0}$. The state equation then reads
\begin{equation*}
	\boldsymbol{x}_t = {A} \boldsymbol{x}_{t-1} + \mathbf{C} z_t + \boldsymbol{\epsilon}_t \label{eq:ESN_state_noise} .
\end{equation*}
The Fischer memory matrix is given by
\begin{equation*}
	\textnormal{F}_{i,j}((z_t)_{t \in\Z_-}) := - \E_{p(\mathbf{x}_t \vert (z_t)_{t \in\Z_-})} \left[ \frac{\partial^2 \log\big( p(\mathbf{x}_t \vert (z_t)_{t \in\Z_-}) \big)}{\partial z_{t-i+1} \, \partial z_{t-j+1}} \,  \right] ,
\end{equation*}
where $p(\boldsymbol{x}_t \vert (z_t)_{t \in\Z_-})$ is the input-conditional state distribution, and the Fischer memory curve is given by its diagonal entries, $\textnormal{F}_\tau \equiv \textnormal{F}_{\tau+1,\tau+1}$ for $\tau \geq 0$. Assuming that $\boldsymbol{\epsilon}_t$, for all $t\in \Z_-$, are mean-zero Gaussian distributed with variance $\sigma^2_\epsilon \mathbb{I}_N$,  one obtains (see the detailed derivations in \citealt{Ganguli2008,Tino2013}) that  $p({\mathbf{x}}_t \vert (z_t)_{t \in\Z_-})$ is Gaussian with the covariance matrix
\begin{equation*}
	R_{\mathbf{x}} = \sigma^2_\epsilon \sum_{j=0}^\infty A^j (A^\top)^j,
\end{equation*}
and hence the FMC can be written as
\begin{equation*}
	\textnormal{F}_\tau = \mathbf{C}^\top (A^\top)^\tau R_{\mathbf{x}}^{-1} A^\tau \mathbf{C}.
\end{equation*}
One may easily notice that this formula does \textit{not} depend on input $\mathbf{z}$ and measures memory based only on the architecture properties of the state-space system.

\subsubsection{Relation Between Memory Capacities and Fischer Memory} 

The relation between these two notions of memory is not straightforward. 
Theorem 1 in \cite{Tino2013} shows that 
\begin{equation*}
	\textnormal{MC}_\tau = \sigma^2_\epsilon \textnormal{F}_\tau + \mathbf{C}^\top (A^\top)^\tau O^{-1} A^\tau \mathbf{C} ,
\end{equation*}
where $O = {\Gamma}_{\mathbf{x}} (R_{\mathbf{x}}/\sigma^2_\epsilon - {\Gamma}_{\mathbf{x}})^{-1} {\Gamma}_{\mathbf{x}} + {\Gamma}_{\mathbf{x}}$, and that it follows $\textnormal{MC}_\tau > \sigma^2_\epsilon \textnormal{F}_\tau$ for all $\tau > 0$. Due to the complex properties of matrix $O$, \cite{Tino2013} does not establish further general results while providing explicit calculations of both MC and FMC when $A$ is symmetric or orthonormal. \cite{tino:symmetric} contains further derivations regarding the asymptotic Fischer memory capacity of particular classes of ESN models.

\medskip

The reasons why in this paper we focus on memory capacity \eqref{eq:MC_definition} instead of the Fischer memory curve are two-fold. First, FMC measures memory only in the state space, and the observation equation \eqref{eq:ESN_def_2} does not have any impact on the FMC computation. Our primary interest is to evaluate memory in terms of real-world applications, which inevitably requires studying the effect of the linear projection of states onto targets encoded by $\mathbf{W}$. Second, important theoretical contributions towards analyzing the impact of noise on the statistical properties of the states and the linear reservoir systems, in general, have already been made in \cite{couillet2016Proc, linearESN}. %In the case of FMC as defined above we would therefore need to precisely study the limit $\sigma^2_\epsilon \to 0$ where $\textnormal{F}_\tau \to \infty$. 
Finally, we emphasize that in Section~\ref{section:subspace}, we are able to show that MC is neutral to the choice of input mask, which is not the case for Fischer memory. As we explain in the following sections, this fact allows us to develop a particular numerical method that is insensitive to numerical artifacts and yields results fully coherent with theory.

\subsection{Linear Models Generically Have Maximal Memory}

Memory capacities of echo state networks with independent inputs have been originally analyzed in \cite{Jaeger:2002}. Already in this work, it was shown that 
\begin{equation*}
	1 \leq \text{MC} \leq N.
\end{equation*}
This statement was extended to more general recurrent neural networks 
in \cite{RC15}, where $N$ is in that case the state space dimension. Two results that have recently appeared in the literature show that linear echo state networks generically achieve \textit{maximal} memory capacity, that is, for almost all LESNs it holds that $\text{MC} = N$. Due to their importance in the sequel of the paper and for the sake of completeness, we collect some of those statements in the following result. The first one is contained in \cite{RC15}, Corollary 4.2, and we reproduce it in the next proposition with an illustrative proof that will be useful for some derivations later on.  

\begin{proposition}[{\bf LESN Memory Capacity}]
	\label{prop:MC}
	Consider a linear ESN model in \eqref{eq:ESN_def_1}-\eqref{eq:ESN_def_2} and let $\boldsymbol{\zeta} = \mathbf{0}$.
	Let $A$ be diagonalizable and such that  $\rho(A) < 1$, with $\rho(A)$ the spectral radius of the matrix $A$. Suppose that all the eigenvalues of $A$ are distinct. Let any of the following equivalent conditions hold
	\begin{description}
		\item[(i)] The vectors $\{{A} \mathbf{C}, {A}^2 \mathbf{C}, \ldots, {A}^N \mathbf{C} \}$ form a basis of $\mathbb{R}^N$.
		\item[(ii)] The Kalman controllability condition holds.
		\item[(iii)] ${A}$ has full rank and $\mathbf{C}$ is neither the zero vector nor an eigenvector of ${A}$.
	\end{description}
	If $({z}_t)_{t\in\mathbb{Z}_-}$ is a weakly stationary white noise process, then $\text{MC} = N$.
\end{proposition}

\begin{proof}
	Under the assumption $\rho(A) < 1$ and of the stationarity and the finite second-order moments of ${\bf z} $, the statement \eqref{limit for finite variance} guarantees that the expression
	\begin{equation*}
		\mathbf{x}_t = \sum_{j=0}^\infty A^j \mathbf{C}\, {z}_{t-j},
	\end{equation*}
	determines almost surely a unique second-order stationary process.
	If $\textnormal{Var}({z}_{t}) = \gamma(0)$, the (time-independent) second moment of the state process is
	\begin{align}
		\label{gamma_x}
		\Gamma_{\mathbf{x}}
		& = \gamma(0)\,\sum_{j=0}^\infty A^j \mathbf{C}\,  \mathbf{C}^\top (A^j) 	^\top
	\end{align}
	and the covariance of the state and the input process is
	\begin{align*}
		\textnormal{Cov}(\mathbf{x}_t, {z}_{t-\tau})
		& = A^\tau \mathbf{C}\, \gamma(0).
	\end{align*}
	Substituting these expressions into \eqref{eq:MC_tau_covvar}, we conclude that the $\tau$-lag memory capacity is
	\begin{equation}\label{eq:MC_tau_AC1}
		\textnormal{MC}_\tau = \mathbf{C}^\top (A^\top)^\tau \left[ \sum_{j=0}^\infty A^j \mathbf{C} \mathbf{C}^\top (A^j) 	^\top \right]^{-1} A^\tau \mathbf{C} , \enspace \tau \in \mathbb{N},
	\end{equation}
	where the inverse is well-defined whenever any of the conditions {\bf (i)}, {\bf (ii)}, or {\bf (iii)} is satisfied (see Proposition 4.3, \cite{RC15}).  
	Summing over all lags and using the fact that $\text{MC}_\tau$ is a scalar, for all $\tau\in \mathbb{N}$, yields
	\begin{align*}
		\textnormal{MC} 
		&= \sum_{\tau = 0}^\infty \mathbf{C}^\top (A^\top)^\tau \left[ \sum_{j=0}^\infty A^j \mathbf{C} \mathbf{C}^\top (A^j) 	^\top \right]^{-1} A^\tau \mathbf{C} \\
		&= \sum_{\tau = 0}^\infty \textnormal{tr}\left( \left[ \sum_{j=0}^\infty A^j \mathbf{C} \mathbf{C}^\top (A^\top)^j \right]^{-1} A^\tau \mathbf{C} \mathbf{C}^\top (A^\tau)^\top \right) \\
		& = \textnormal{tr}\left( \left[ \sum_{j=0}^\infty A^j \mathbf{C} \mathbf{C}^\top (A^j) 	^\top \right]^{-1} \sum_{\tau = 0}^\infty A^\tau \mathbf{C} \mathbf{C}^\top (A^\tau)^\top \right)  = \textnormal{tr}(\mathbb{I}_N)  = N,
	\end{align*}
	as required.
\end{proof}

The second important result was originally presented in \cite{RC21}\footnote{The authors of \cite{RC21} acknowledge that the proof of the proposition has been communicated to them by Friedrich Philipp.} and guarantees that the conditions {\bf (i)}-{\bf (iii)} in Proposition~\ref{prop:MC} are satisfied almost surely, whenever, as it is customary in reservoir computing, the connectivity matrix $A$ and the input mask $\mathbf{C}$ of the linear system \eqref{eq:ESN_def_1}-\eqref{eq:ESN_def_2} are randomly drawn from some regular probability distribution. We recall that a random variable $X : \Omega \rightarrow \mathbb{R} $ defined on a probability space $(\Omega, \mathcal{F}, \mathbb{P}) $ and with values on a Borel measurable space $\mathbb{R} $ is {\it regular} whenever $\mathbb{P} \left(X =a\right)=0 $ for all $a \in \mathbb{R} $. The result is stated in the following proposition.

\begin{proposition}[\cite{RC21}]
	\label{random_matrix_lemma}
	Let $N \in \mathbb{N}  $,  $A\in \mathbb{M}_{N}$, and $\mathbf{C}\in \mathbb{R}^N$  and assume that the entries of $A$ and $\mathbf{C}$ are drawn using independent regular real-valued distributions. Then the following statements hold:
	\begin{description}
		\item [(i)]  The vectors $\{\mathbf{C}, A\mathbf{C}, A^2 \mathbf{C}, \ldots , A^{N-1}\mathbf{C}\}$ are linearly independent almost surely.
		\item [(ii)]  Given $m$ distinct complex numbers $\lambda_1, \ldots, \lambda_m \in \mathbb{C}$, where $m \leq N$, the event that $1, \lambda_1 , ... , \lambda_m \notin \sigma(A)$ ($\sigma(A)$ is the spectrum of $A$) and that the vectors
		\begin{align*}
			(\mathbb{I} - \lambda_j A)^{-1}(\mathbb{I} - A)^{-1}(\mathbb{I} - A^N)\mathbf{C}, \quad \mbox{$j = 1, \dots , m$}
		\end{align*}
		are linearly independent holds almost surely.
	\end{description}
\end{proposition}

Proposition~\ref{prop:MC} (see also Corollary 4.2 in \citealt{RC15})  together with Proposition~\ref{random_matrix_lemma} give a definite answer to the question of whether some reservoir architectures have theoretically more memory capacity than others in the linear setting. In theory, as we just showed, all linear ESNs with the connectivity matrix $A$ and the input mask $\mathbf{C}$ drawn from regular distributions {\it achieve almost surely their upper memory bound} regardless of the underlying design under minimal algebraic conditions. This means that the discussions about optimizing the components of a LESN reservoir's state map to achieve maximal theoretical memory capacity are not justified. Regardless of the LESN architecture, in the setting of Proposition~\ref{prop:MC}, {\it the memory of any LESN is generically maximal}. 

However, empirical estimates in the literature of the memory capacity in applied memory tasks may differ from the theoretical value of $N$, which has motivated multiple studies with attempts to design LESNs that render ``maximized'' memory. In the following sections, we characterize the problems associated with the most common ways of numerical estimation of MCs and explain computational issues that yield misleading empirical results. We show that purely numerical pathologies in empirical MC evaluation emerge in a plethora of memory ``maximization'' techniques applied to LESN architectures. As a solution, we shall propose a simple numerical scheme to combat the numerical inconsistency of empirical estimates with the theoretical result in Proposition~\ref{prop:MC}.

\subsection{Monte Carlo Estimation of Memory Capacities}
\label{Monte Carlo estimation of memory capacities}
In this section, we address important issues that arise when estimating memory capacities using standard Monte Carlo simulation tools. The definitions and the results discussed in this section show that even in the simplified setting of the so-called regular linear systems, the simulation-based estimation of network capacities may be misleading. We use this section exclusively to motivate the necessity of designing other numerical methods for capacity estimation that do not suffer from the poor statistical properties of na\"ive approaches based on plug-in estimators. In the following paragraphs, we spell out the finite-sample properties of the natural sample estimator of (total) memory capacity and illustrate the limitations of the sample-based approach that may lead to incorrect memory estimates that are incompatible with the generic $N$-memory capacity of LESNs.

{The availability of the closed-form solution \eqref{eq:MC_tau_covvar} facilitates the computation of capacities. However, even for linear specifications, the ill-conditioning of the associated covariance matrices of states leads to technical difficulties.  Some of those problems can be handled by using equivalent state-space representations. Proposition~2.5 in \cite{RC15} proves that new representations obtained out of linear injective system morphisms leave capacities invariant and hence can be used to produce systems with more technically tractable properties.} 

\begin{proposition}[{\bf Standardization of state-space realizations,} \citealt{RC15}]
	\label{std rc io}
	~\\Consider a state-space system as in \eqref{eq:ESN_def_1}-\eqref{eq:ESN_def_2} and suppose that $\rho(\widetilde{A})<1$. Let ${\bf z}: \Omega\longrightarrow \mathbb{R}^{{\mathbb{Z}}_{-}}$ be a stationary mean-zero input process  and let $\widetilde{{\bf x}}: \Omega\longrightarrow (\mathbb{R}_N)^{{\mathbb{Z}}_{-}}$ be the associated stationary state process given by \eqref{limit for finite variance}.  Suppose that the covariance matrix  $\Gamma_{\widetilde{{\bf x}}} := {\rm Cov}(\widetilde{\mathbf{x}}_t, \widetilde{\mathbf{x}}_t)$ is non-singular.
	Then, the map $f: \mathbb{R}^N \longrightarrow \mathbb{R}^N$ given by $f(\widetilde{\mathbf{x}}):=\Gamma_{\widetilde{{\bf x}}}^{-1/2} \widetilde{\mathbf{x}} $ is a system isomorphism between the system \eqref{eq:ESN_def_1}-\eqref{eq:ESN_def_2} and the one with state map 
	\begin{equation}
		\label{isomorphic state eq}
		\widetilde{F}( \mathbf{x}, {z}): = A \mathbf{x} + \mathbf{C} { z}
	\end{equation}
	and readout
	\begin{equation}
		\label{isomorphic readout eq}
		\widetilde{h}( \mathbf{x}): = {\mathbf{W}}^\top \mathbf{x},
	\end{equation}
	with $A:=\Gamma _{{\bf x}}^{-1/2} \widetilde{A}\Gamma_{{\bf x}}^{1/2}$, $\mathbf{C}:=\Gamma _{{\bf x}}^{-1/2}\widetilde{\mathbf{C}}$, and ${\mathbf{W}}=\Gamma _{{\bf x}}^{-1/2} \widetilde{\mathbf{W}}$.
	Moreover, the state process ${{\bf x} } $ associated to the system $\widetilde{F} $  and the input ${\bf z} $ is covariance stationary and
	\begin{equation}
		\label{mean and covariance new states}
		{\rm E}[{\mathbf{x}}_t] = {\bf 0}, \quad \mbox{and} \quad {\rm Cov}( {\mathbf{x}}_t, {\mathbf{x}}_t) = \mathbb{I}_N.
	\end{equation}
\end{proposition}

This result of invariance of memory capacities with respect to the system isomorphism $f(\mathbf{x}):=\Gamma_{{\bf x}}^{-1/2} \mathbf{x} $ allows us to work directly with the standardized state-space systems and assume that $\Gamma_{\mathbf{x}} = \mathbb{I}_N$ without loss of generality. 

\begin{definition}
	\label{reg_lin_def}
	Let  ${\bf z}: \Omega\longrightarrow \mathbb{R}^{{\mathbb{Z}}_{-}}$, $D \subset \mathbb{R}$, be a variance-stationary input process and let the state map $\widetilde{F}:\mathbb{R}^N\times \mathbb{R}\rightarrow \mathbb{R}^N$ be given by $\widetilde{F}(\mathbf{x},z):=A\mathbf{x}+\mathbf{C}z$ with $\rho(A)<1$. We call a system with the state map $\widetilde{F}$  a {\bfi  regular linear system} whenever the covariance matrix $\Gamma_{\mathbf{x}}$ of the associated covariance-stationary state process ${\bf x}: \Omega\longrightarrow (\mathbb{R}^N)^{{\mathbb{Z}}_{-}}$ satisfies $\Gamma_{\mathbf{x}} = \mathbb{I}_N$.
\end{definition}

A straightforward approach to estimate the memory capacity of an echo state network is to simulate the mean zero and variance one process $({z}_t)_{t=1}^T$, to compute the associated states $(\mathbf{x}_t)_{t=1}^T$
%  using a trivial initialization $\mathbf{x}_0=0 $,\footnote{For regular linear systems $\rho(A)<1$ by Definition~\ref{reg_lin_def}, and hence the choice of the starting value $\mathbf{x}_0$ is arbitrary.} 
and to use in \eqref{eq:MC_tau_covvar} the  plug-in sample estimator
%In view of equation (\ref{eq:MC_tau_covvar})
\begin{align}
	\widehat{\boldsymbol{\gamma}_{\mathbf{x}z}} (\tau):=\widehat{\textnormal{Cov}}(\mathbf{x}_t, {z}_{t-\tau})
	& = \frac{1}{T - \tau} \sum_{t=\tau+1}^T \mathbf{x}_t\, {z}_{t-\tau}\label{gamma_xz}.
\end{align}
This leads to the {sample memory capacity estimator}
\begin{equation}\label{eq:MC_tau_sample_estimator}
	\widehat{\textnormal{MC}}_\tau := 
	{\widehat{\boldsymbol{\gamma}_{\mathbf{x}z}}  (\tau)^\top  \widehat{\boldsymbol{\gamma}_{\mathbf{x}z}} (\tau)}
	=\big\lVert \widehat{\boldsymbol{\gamma}_{\mathbf{x}z}} (\tau) \big\rVert_2^2,
\end{equation}
that we refer to as the Monte Carlo estimator. Letting $N$ be fixed, under suitable assumptions of stationarity and sufficiently many finite moments, it is well-known that the above sample estimators are consistent and asymptotically normal (see \citealt{BrocDavisYellowBook, MR1278033} and \citealt{luetkepohl:book} for the details). These assumptions hold trivially when ${z}_t$ is sampled as i.i.d. standard Gaussian noise and we show further that $\widehat{\textnormal{MC}}_\tau \overset{p}{\to} \textnormal{MC}_\tau$ as $T \to \infty$ for any fixed $\tau$. 

However, if $N$ is growing with $T$, $\widehat{\textnormal{Cov}}(\mathbf{x}_t, {z}_{t-\tau})$ may be inconsistent. In practical implementations of echo state network architectures $N$ can be large, of the order of $10^4$ or more. Hence, $T$ must also be appropriately chosen for the Monte Carlo approximations to be valid. These considerations imply the necessity to study memory estimators in the high-dimensional time series setting (see e.g. \citealt{chen2013covariance} or \citealt{zhang2017gaussian} for examples of such discussions). We show in the following paragraphs that inaccuracies when numerically evaluating $\widehat{\textnormal{MC}}_\tau$ even when the ratio $T / N$ is small (in practice $T / N < 10$ can already be problematic) mean that the estimator \eqref{eq:MC_tau_sample_estimator} is a poor approximation of the LESN memory capacity. This issue can be even more significant when one wishes to quantify $\textnormal{MC}_\tau$ for $\tau$ large. We now provide a quick analysis of these phenomena using standard statistical arguments.

\begin{proposition}
	\label{memory_capacity_estimator} 
	Let $N\in \mathbb{N}$, $A\in \mathbb{M}_{N}$, $\mathbf{C}\in \mathbb{R}^N$, and $\boldsymbol{\zeta} = \mathbf{0}$, and suppose that the resulting linear system is regular. 
	Let $({z}_t)_{t=1}^T$, $T\in \mathbb{N}$, be mean-zero i.i.d. Gaussian with ${\text{Var}}({z}_t) = \gamma(0)=1$ and let $(\mathbf{x}_t)_{t=1}^T$ be the associated states (obtained using a trivial initialization). 
	Then the  memory capacity sample estimator for any $\tau\in \mathbb{N}$, $\tau<T$, is given by
	\begin{equation}
		\label{MCtauest}
		\widehat{\textnormal{MC}}_\tau(T) = \big\lVert \widehat{\boldsymbol{\gamma}_{\mathbf{x}z}} (\tau) \big\rVert_2^2
	\end{equation}
	with $\widehat{\boldsymbol{\gamma}_{\mathbf{x}z}} (\tau)$ as in \eqref{gamma_xz}
	and the total memory capacity  estimator on $\tau_{\rm{max}}\in \mathbb{N}$ sample of memory capacities is
	\begin{equation}
		\label{MCest}
		\widehat{\textnormal{MC}}(T) = \frac{1}{\tau_{\rm{max}}}\sum_{\tau=0}^{\tau_{\rm{max}}-1}\widehat{\textnormal{MC}}_\tau(T).
	\end{equation}
	These estimators have the following properties:
	\begin{description}
		\item [(i)]  $\widehat{\textnormal{MC}}_\tau(T)$ is a biased estimator of ${\textnormal{MC}}_\tau$ with bias $B_{MC} $ given by
		\begin{align}
			\label{bias}
			B_{MC}:=\E[\widehat{\textnormal{MC}}_\tau(T)] - {\textnormal{MC}}_\tau = \frac{N}{T - \tau} + \frac{2}{T - \tau} \sum_{j=0}^{\tau} {\boldsymbol{\gamma}_{\mathbf{x}z}} (j) ^\top {\boldsymbol{\gamma}_{\mathbf{x}z}} (2\tau-j),
		\end{align}
		which is positive for large $\tau$.
		\item [(ii)]  $\widehat{\textnormal{MC}}_\tau(T)$ is an asymptotically unbiased estimator, that is $(\E[\widehat{\textnormal{MC}}_\tau(T)] - {\textnormal{MC}}_\tau) \rightarrow 0$ as $T\rightarrow \infty$, and a weakly consistent estimator of ${\textnormal{MC}}_\tau$, that is $\widehat{\textnormal{MC}}_\tau(T) \overset{p}{\to} \textnormal{MC}_\tau$ as $T \to \infty$. 
		\item [(iii)]  $\widehat{\textnormal{MC}}(T)$ is a biased and asymptotically unbiased estimator of ${\textnormal{MC}}$. Moreover, it is weakly consistent, that is $\widehat{\textnormal{MC}}(T) \overset{p}{\to} \textnormal{MC}$ with $\tau_{\rm{max}} =O(T)$ and $T \to \infty$.
	\end{description}
\end{proposition}

\begin{proof} 
	
	Both \eqref{MCtauest} and \eqref{MCest} are immediate consequences of the definition \eqref{eq:MC_tau_sample_estimator}. To show {\bf (i)}, we use that by Definition \ref{reg_lin_def}, $\rho(A)<1$  and obtain that 
	\begin{align*}
		\E \big[  \big\lVert \widehat{\boldsymbol{\gamma}_{\mathbf{x}z}(\tau)} \big\rVert_2^2 \big]
		=  &\E\left[ \left( \frac{1}{T - \tau} \sum_{t=\tau+1}^T \mathbf{x}_t\, {z}_{t-\tau} \right)^\top \left( \frac{1}{T - \tau} \sum_{s=\tau+1}^T \mathbf{x}_s\, {z}_{s-\tau} \right) \right] \\
		= &\frac{1}{(T - \tau)^2}\E\left[ \left( \sum_{t=\tau+1}^T \sum_{j=0}^{\infty}A^j \mathbf{C} \, {z}_{t-j} \, {z}_{t-\tau} \right)^\top \left( \sum_{s=\tau+1}^T \sum_{k=0}^{\infty}A^k \mathbf{C} \, {z}_{t-k} \, {z}_{s-\tau} \right) \right] \\
		= &\frac{1}{(T - \tau)^2} \, \E\left[ \mathbf{C}^\top \left\{\sum_{t=\tau+1}^T \sum_{s=\tau+1}^{T} \sum_{j=0}^{\infty} \sum_{k=0}^{\infty} (A^j)^\top A^k  {z}_{t-\tau} \, z_{t-j}  {z}_{s-\tau} z_{s-k}   \right\}\mathbf{C} \right] \\
		= &\frac{1}{(T - \tau)^2} \Bigg\{\, \sum_{t=\tau+1}^T \mathbf{C}^\top (A^\tau)^\top A^\tau \mathbf{C}\, \E\left[  {z}_{t-\tau}^4    \right]+ \sum_{t\neq s} \mathbf{C}^\top (A^\tau)^\top A^\tau \mathbf{C}\, \E\left[  {z}_{t-\tau}^2 \, z_{s-\tau}^2    \right]\\
		&+\sum_{t=\tau+1}^T \sum_{j=0, j\neq \tau}^{\infty} \mathbf{C}^\top (A^j)^\top A^j \mathbf{C}\, \E\left[  {z}_{t-\tau}^2 \, z_{t-j}^2    \right] \\
		&+\sum_{t=\tau+1}^T \sum_{j=0, j\neq \tau}^{2\tau} \mathbf{C}^\top (A^j)^\top A^{2\tau-j} \mathbf{C}\, \E\left[  {z}_{t-\tau}^2 \, z_{t-j}^2    \right]
		\Bigg\}\\
		=& \frac{1}{T - \tau}\big\{3\|\boldsymbol{\gamma}_{\mathbf{x}z} (\tau)\|_2^2 + (T - \tau - 1)\|\boldsymbol{\gamma}_{\mathbf{x}z} (\tau)\|_2^2 +  \gamma(0)\textnormal{tr}(\Gamma_{\mathbf{x}})-\|\boldsymbol{\gamma}_{\mathbf{x}z} (\tau)\|_2^2\\
		&+\sum_{j=0}^{2\tau} \mathbf{C}^\top (A^j)^\top A^{2\tau-j} \mathbf{C} \gamma(0)^2-\|\boldsymbol{\gamma}_{\mathbf{x}z} (\tau)\|_2^2\big\}\\
		=&\|\boldsymbol{\gamma}_{\mathbf{x}z} (\tau)\|_2^2 + \frac{1}{(T - \tau)}  \gamma(0)\textnormal{tr}(\Gamma_{\mathbf{x}})+ \frac{1}{(T - \tau)} \sum_{j=0}^{2\tau} \boldsymbol{\gamma}_{\mathbf{x}z} (j)^\top \boldsymbol{\gamma}_{\mathbf{x}z} (2\tau-j),
	\end{align*}
	where we can use that $\Gamma_{\mathbf{x}} = \mathbb{I}_N$ and that $\gamma(0)=1$, which yields \eqref{bias}.
	
	Further, using that for any $\epsilon>0$ there exists a matrix norm $\vertiii{\cdot}$ such that $\vertiii{A}=\rho(A)+\epsilon$ (see Lemma 7.6.12 in \citealt{horn:matrix:analysis}), the second term in $B_{MC}$ can be bounded as follows
	\begin{align*}
		\frac{1}{T - \tau} \sum_{j=0}^{2\tau} \vert \boldsymbol{\gamma}_{\mathbf{x}z} (j)^\top \boldsymbol{\gamma}_{\mathbf{x}z} (2\tau-j) \vert&=  \frac{\gamma(0)^2}{T - \tau}  \sum_{j=0}^{2\tau} \vert\mathbf{C}^\top (A^j)^\top A^{2\tau-j} {\mathbf{C}} \vert \\&\leq \frac{\gamma(0)^2}{T - \tau} \sum_{j=0}^{2\tau} \vertiii{\mathbf{C}}^2 \vertiii{(A^\top)^j} \vertiii{A^{2\tau-j}} \\
		&\leq \frac{\gamma(0)^2}{T - \tau}  \vertiii{\mathbf{C}}^2\sum_{j=0}^{2\tau}  \vertiii{A^\top}^j \vertiii{A} ^{2\tau-j}\\
		&=\frac{\tau+1}{T - \tau} \gamma(0)^2 \vertiii{\mathbf{C}}^2    (\rho(A) + \epsilon) ^{2\tau},
	\end{align*}
	and hence decays exponentially fast with $\tau$. It is also easy to see that  $B_{MC}$ is always positive for large enough $\tau$.
	
	In order to show {\bf{(ii)}}, notice that  {\bf{(i)}} together with the Markov inequality gives $\widehat{\text{MC}}_\tau(T) = O_p(T^{-1})$ which yileds asymptotic unbiasedness as $T\rightarrow \infty$ and weak consistency of $\widehat{\textnormal{MC}}_\tau(T)$ as an estimator of ${\textnormal{MC}}_\tau$.
	Finally, in {\bf{(iii)}} one can mimic the proof of {\bf{(ii)}} and use that, 
	by Gelfand's formula \citep{lax:functional:analysis}, $\lim\limits_{k \rightarrow \infty}\vertiii{A ^k}^{1/k}=\rho(A)<1$, which implies the existence of a number $k _0 \in \mathbb{N}$ such that $\vertiii{A ^k}<1 $, for all $k\geq k _0 $. Consequently, this implies the finiteness of all the sums in $\widehat{\text{MC}}_\tau(T)$ with $\tau_{\rm{max}} =O(T)$ and $T \to \infty$. 
\end{proof}

This result shows that even though the estimator of the memory capacity is asymptotically unbiased, in finite samples $\textnormal{MC}_\tau$ is always positively biased above zero. This means that, even with large $T$ summing up $\tau_{\rm max}$ terms in the sequence $\{\text{MC}_\tau\}^{\infty}_{\tau=0}$ may yield a memory capacity estimate that is above the theoretical limit given by Proposition~\ref{prop:MC}. This happens even when reservoir matrices are well-conditioned. Figure~\ref{fig:sample_memory_inflation} illustrates the case when $N = 100$, $A$ is a scaled random orthogonal matrix, and $\mathbf{C}$ is a $2$-norm-scaled random normal vector. Taking $\tau_{\rm max} = 500$ to estimate MC, we show that even in those Monte Carlo simulations where $T/N \approx 100$ non-negligible memory overestimation errors are committed.
% chacnge the oy axis title in the right figure $\widehat{\textnormal{MC}}(T) / N$
% 

\begin{figure}[t]
	\centering
	\includegraphics[width=0.95\textwidth]{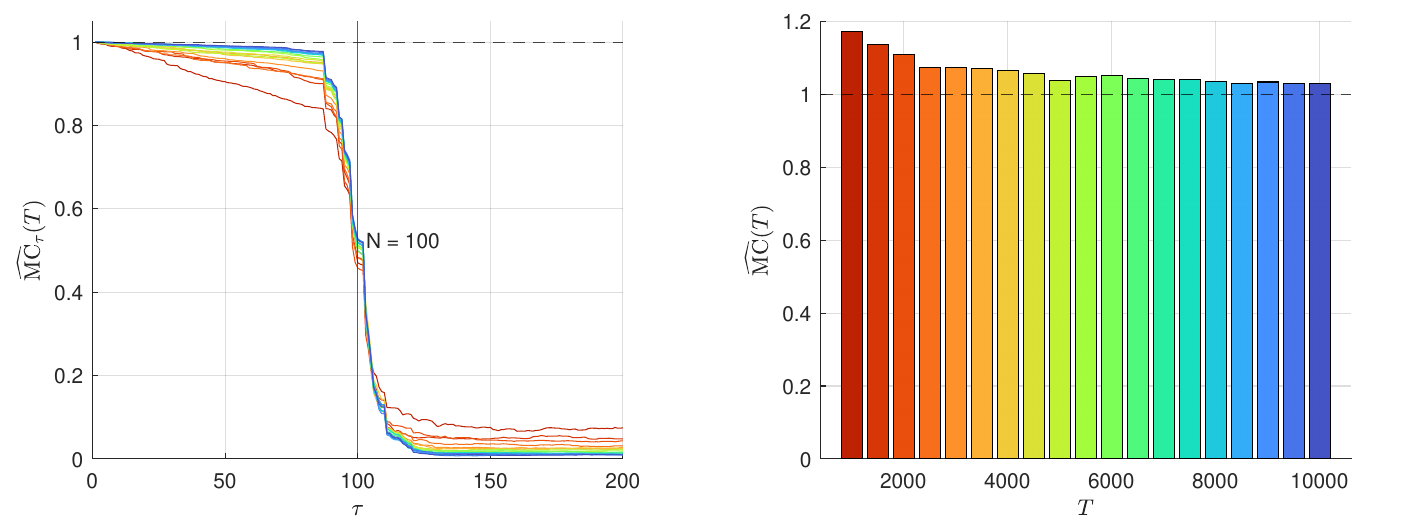}
	\caption{Illustration of memory capacity inflation due to the inconsistent estimation of $\textnormal{MC}_\tau$ for  LESN with $N = 100$, orthogonal $A$ with $\rho(A) = 0.9$, and input mask $\mathbf{C}= \overline{\mathbf{C}} / \norm{\overline{\mathbf{C}}}$ with $\mathbf{C} = (\overline{c}_{i})_{i=1}^N \sim\ \text{i.i.d.}\ \mathcal{N}(0,1)$: (a) memory curves $\widehat{\textnormal{MC}}_\tau(T)$; (b) bar chart of normalized total memory capacity $\widehat{\textnormal{MC}}(T) / N$. 
		Memory curves $\widehat{\textnormal{MC}}_\tau(T)$ are computed for $\tau \in \{0, 1, ..., 5N\}$ (in (a), $\widehat{\textnormal{MC}}_\tau(T)$  is plotted only up to $\tau=2N $ for the sake of clarity). Estimators are computed from simulated $(z_t)_{t=1}^T \sim \text{i.i.d.}\ \mathcal{N}(0,1)$,  with $T\in \{1000, 1500,\ldots, 10000\}$. 
	}
	\label{fig:sample_memory_inflation}
\end{figure}

\subsection{Na\"ive Algebraic Memory Estimation}
\label{Naive Algebraic Method}
In this section, we consider another possibility for the evaluation of the memory using a purely algebraic approach and without relying on Monte Carlo simulations. Again, we show that numerical issues are also encountered with this approach. We start by noticing that, under the hypotheses in Proposition~\ref{prop:MC}, the memory capacity can be computed using \eqref{eq:MC_tau_AC1}, namely, for any $\tau \in \mathbb{N}$
\begin{equation}
	\label{eq:MC_tau_AC}
	\textnormal{MC}_\tau =\mathbf{C}^\top (A^\tau)^\top {G}_{\mathbf{x}}^{-1} A^\tau \mathbf{C},
\end{equation}
where ${G}_{\mathbf{x}} := \gamma(0)^{-1} \Gamma_{\mathbf{x}} $ denotes the normalized version of the state autocovariance matrix in \eqref{gamma_x} and can be written as
\begin{equation}\label{eq:Gx_series}
	{G}_{\mathbf{x}} = \sum_{j=0}^\infty A^j \mathbf{C} \mathbf{C}^\top (A^j)^\top.
\end{equation}
The infinite series in the definition of ${G}_{\mathbf{x}}$ may be hard to approximate well with a finite number of terms if the spectral radius of ${A}$ is very close to one, a choice that is quite common in applications. This concern can be easily mitigated by noting that under the hypotheses of Proposition~\ref{prop:MC} a closed-form expression of ${G}_{\mathbf{x}}$ in terms of the eigendecomposition of ${A}$ can be derived (for details, we refer the reader to the proof of Proposition~4.3 in \cite{RC15}). 
Let $\{\mathbf{v}_1, \ldots, \mathbf{v}_N\}$ be an eigenbasis of ${A}$ and $\{\lambda_1, \ldots, \lambda_N\}$ be the associated eigenvalues. By expressing $\mathbf{C}$ as $\mathbf{C} = \sum_{i=1}^N {c}_i \, \mathbf{{v}}_i$ it is straightforward to show that
\begin{equation}\label{eq:Gx_eigenbasis}
	{G}_{\mathbf{x}} = \sum_{i,j = 1}^N \frac{{c}_i \bar{{c}_j}}{1 - \lambda_i \bar{\lambda}_j} \, \mathbf{v}_i \, \mathbf{v}_j^\ast.
\end{equation}
and hence ${G}_{\mathbf{x}}$ can be readily and precisely computed.  Unfortunately, ${G}_{\mathbf{x}}$ can still be significantly poorly conditioned for moderately large $N$ and commonly chosen distributions for the entries of $A$. This problem is easy to illustrate by plotting the norm of the eigenvalues of ${G}_{\mathbf{x}}$ for  $A$ sampled from laws that are standard in the literature. 

We provide an example demonstrating this phenomenon in Figure~\ref{fig:G_x_eigenvalues} using two reservoirs of size $N = 50$ and $N = 150$. More precisely, for five commonly used choices of connectivity matrices, we plot the absolute values of the eigenvalues of $G_{\mathbf{x}}$ in decreasing order. We compare them with the standard double-precision of floating point numbers {\it eps} in our software of choice, MATLAB. Notice that all the eigenvalues in absolute value smaller than {\it eps} will be numerically treated as zero by linear algebra routines. This poor conditioning does not by itself mean that software packages will fail to solve the linear system given by ${G}_{\mathbf{x}} \, \mathbf{u} = A^\tau \mathbf{C}$; rather, the numerical solution for $\mathbf{u}$ will be inaccurate \cite[Section 5.8]{horn:matrix:analysis}. 
\begin{figure}
	\centering
	\begin{subfigure}[b]{0.49\textwidth}
		\centering
		\includegraphics[width=\textwidth]{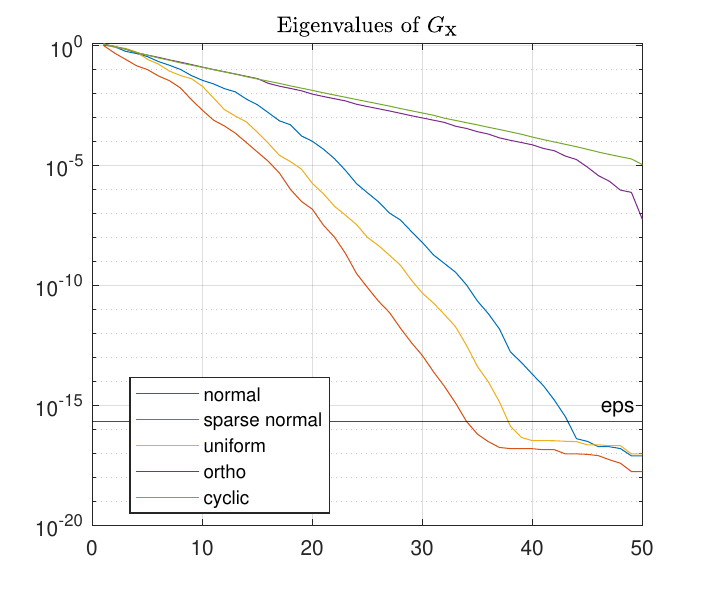}
		\caption{$N = 50$}
		\label{}
	\end{subfigure}
	%\hfill
	\begin{subfigure}[b]{0.49\textwidth}
		\centering
		\includegraphics[width=\textwidth]{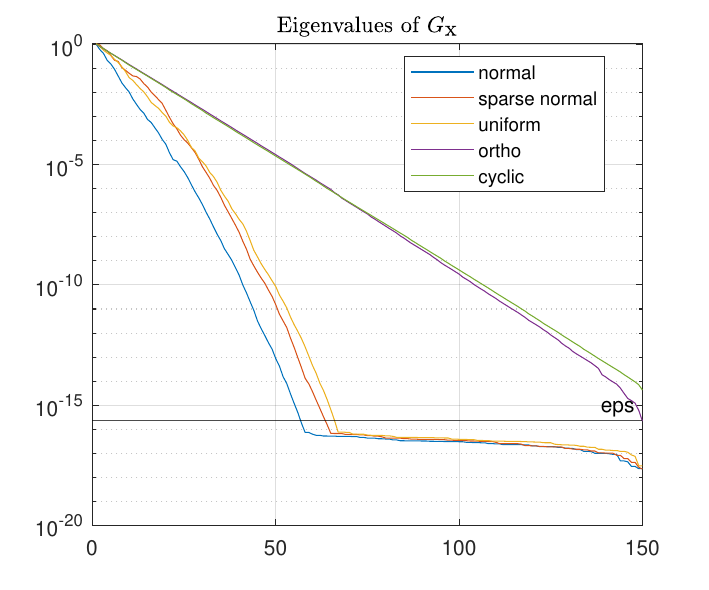}
		\caption{$N = 150$}
		\label{}
	\end{subfigure}
	\caption{Eigenvalue plot (in absolute values) for ${G}_{\mathbf{x}}$ for various types of connectivity matrices.  ${G}_{\mathbf{x}}$ was computed using $1000$ series terms in \eqref{eq:MC_tau_AC}, a connectivity matrix $A \in \mathbb{M}_N$ with spectral radius $\rho(A) = 0.9$ and a unit norm input mask $\mathbf{C} \in \mathbb{R}^N$. Computations are performed in MATLAB with the standard double-precision of floating point numbers $eps =2^{-52}\approx 2.2\times 10^{-16}$ marked with the black horizontal solid line.}
	\label{fig:G_x_eigenvalues}
\end{figure}
This numerical instability is at the origin of the seemingly suboptimal memory performance of LESNs observed in implementations. Further, note that even if $\Gamma_{\mathbf{x}}$ is estimated using a Monte Carlo simulation, due to its consistency as $T \to \infty$, the sample estimator $\widehat{\Gamma}_{\mathbf{x}}$ inherits the conditioning issues of its theoretical counterpart. This effectively implies that the simulation of ${G}_{\mathbf{x}}$ is not a feasible way to mitigate the conditioning problem, even asymptotically. Regularization methods, such as Tikhonov, also do not solve this as they modify the covariance eigenvalue structure.

The following example about the so-called \textit{cyclic reservoirs} is much studied in the literature under the name ``RingOfNeurons'' (see \citealt{ogESN2012}, for example, or more recently in \citealt{verzelliInputtoStateRepresentationLinear2021}). Cyclic architectures yield memory curves that are computable in closed form, and the ill-conditioning of ${G}_{\mathbf{x}}$ can be explicitly demonstrated. Cyclic ESNs fall into the more general category of orthogonal recurrent neural networks, for which \cite{White2004} has also derived some theoretical memory properties. However, there the authors consider the case in which the states may be contaminated by noise, a situation that we do not discuss in this work.

\begin{example}[{\bf Cyclic reservoirs}] \normalfont
	Consider a $N$-dimensional cyclic reservoir with the unscaled orthogonal connectivity matrix 
	\begin{equation*}
		\widetilde{A} = \left(\begin{array}{ccccc}
			0 & 0 &\ldots & 0 & 1 \\
			1 & 0 & \ddots & 0 & 0 \\
			0 & \ddots & \ddots & \vdots & \vdots \\
			\vdots  & \ddots & \ddots & 0 & 0 \\ 
			0  & \ldots & 0 & 1 & 0 
		\end{array}\right) \in \mathbb{M}_N,
	\end{equation*}
	which is rescaled with some $\rho_A < 1$ by setting $A = \rho_A \widetilde{A}$.

	In the literature, both $\widetilde{A}$ and $\widetilde{A}^\top$ are referred to as the cyclic reservoir matrices, and the state dynamics they define are identical up to a permutation of the reservoir nodes \citep{Rodan2011}. Let $\mathbf{C} = \mathbf{e}_1$ be the first canonical basis vector of $\mathbb{R}^N$. First, observe that
	\begin{equation*}
		A \mathbf{C} = \mathbf{e}_2,
		\quad 
		A^2 \mathbf{C} = \mathbf{e}_3,\,
		\ldots,
		\,
		A^{N-1} \mathbf{C} = \mathbf{e}_N,
	\end{equation*}
	which justifies the use of the term {\it cyclic}.\footnote{Further, in Proposition~\ref{prop:input_mask_neutral} we prove that memory capacities $\text{MC}_\tau$ are invariant with respect to the choice of $\mathbf{C}$. Hence, our selection of input mask does not imply any loss of generality.}
	Second, note that in this case we can  obtain the explicit expression of the normalized state covariance matrix as follows:
	\begin{align*}
		{G}_{\mathbf{x}} 
		& =
		\text{diag}\Bigg( \sum_{j=0}^\infty \rho_A^{i(2N)}, \sum_{j=0}^\infty \rho_A^{i(2N)+2}, \ldots , \sum_{j=0}^\infty \rho_A^{i(2N)+2(N-1)} \Bigg)  =\\
		&\text{diag}\Bigg( \frac{1}{1 - \rho_A^{2N}}, \frac{\rho_A^2}{1 - \rho_A^{2N}}, \ldots , \frac{\rho_A^{2(N-1)}}{1 - \rho_A^{2N}} \Bigg) 
	\end{align*}
	and hence
	\begin{equation*}
		{G}_{\mathbf{x}}^{-1} 
		= 
		\text{diag}\left( 1 - \rho_A^{2N},\ \frac{1 - \rho_A^{2N}}{\rho_A^2} ,\ \ldots\ ,\ \frac{1 - \rho_A^{2N}}{\rho_A^{2(N-1)}} \right) .
	\end{equation*}
	This formula shows that if $N$ is large, inversion of ${G}_{\mathbf{x}}$ can be an ill-conditioned problem depending on $\rho_A$. Finally, for $0 \leq \tau \leq N-1$ it holds
	\begin{equation*}
		\text{MC}_\tau 
		= \mathbf{e}_1^\top (A^\tau)^\top {G}_{\mathbf{x}}^{-1} A^\tau \mathbf{e}_1 
		= \rho_A^\tau \left( \frac{1 - \rho_A^{2N}}{\rho_A^{2\tau}} \right) \rho_A^\tau
		= 1 - \rho_A^{2N} ,
	\end{equation*}
	while in general for $k N \leq \tau \leq k (N+1) - 1$, $k > 1$, one has
	\begin{equation*}
		\text{MC}_\tau 
		= \rho_A^{kN + \tau} \left( \frac{1 - \rho_A^{2N}}{\rho_A^{2\tau}} \right) \rho_A^{kN + \tau}
		= \rho_A^{2 kN} (1 - \rho_A^{2N}) .
	\end{equation*}
	These computations are a special case of more general results in \cite{Rodan2011}, although we have made explicit the values of $\text{MC}_\tau$. \cite{Rodan2011} further proved that such memory capacities arise for generic $\mathbf{C}$ when $A$ is chosen to be a regular rotation based on the input mask.
\end{example}

\section{Robust Memory Computation}
\label{section:subspace}

In this section, we propose simple but effective methods to compute the memory capacity $\text{MC}_\tau$ for linear ESNs. These methods are not affected by the problems discussed in Section~\ref{section:naive}. First, we show a strong neutrality result of the memory capacity with respect to input masks. Second, we discuss the origin of numerical instabilities of memory computation and the so-called memory gaps, borrowing from the theory of Krylov subspaces. Finally, we propose new computational methods based on the Arnoldi iteration algorithm for the leading eigenvector computation and on the memory neutrality with respect to the input mask.  We call our proposed methods \textit{robust}, since they do not suffer explicitly from the conditioning issues that arise in the na\"ive algebraic and statistical methods presented in the previous section and render empirical results that are in agreement with the theory.

\subsection{Input Mask Memory Neutrality}

A fundamental aspect of memory capacity is its dependence on the structure of the connectivity matrix $A$ and the input mask $\mathbf{C}$. We recall that by Proposition~\ref{prop:MC}, we know that as long as $A$ and $\mathbf{C}$ satisfy a controllability condition, the total memory MC of a LESN is maximal. Moreover, by Proposition~\ref{random_matrix_lemma} this holds almost surely whenever both  $A$ and $\mathbf{C}$ are sampled from some regular distribution. We now prove a much stronger result: in the linear setup, under the same controllability conditions, the input mask $\mathbf{C}$ does not have any impact on individual $\tau$-lag memory capacities $\text{MC}_\tau$.

\begin{proposition}[{\bf Input mask neutrality}]
	\label{prop:input_mask_neutral}
	For any linear echo state network under the assumptions of Proposition~\ref{prop:MC}, the {memory capacity is input mask neutral}, that is, $\textnormal{MC}_\tau$ is invariant with respect to the choice of $\mathbf{C}$, for all $\tau\in \mathbb{N}$.
\end{proposition}

\begin{proof}
	Let $\{\mathbf{v}_1, \ldots, \mathbf{v}_N\}$ be an eigenbasis of ${A}$ and $\{\lambda_1, \ldots, \lambda_N\}$ be the associated eigenvalues. Denote $\Lambda := \text{diag}(\lambda_1, \ldots, \lambda_N)$, $V := (\mathbf{v}_1|\mathbf{v}_2|\dots|\mathbf{v}_N)$, and $$V^{-1} = \left(\begin{array}{c}{\mathbf{{v}}}_1^\ast\\ \vdots\\ \mathbf{{v}}_N	^\ast\end{array}\right),$$ and notice that by the hypothesis of diagonalizability of $A$ one has $A = V \Lambda V^{-1}$. Using the eigenbasis of $A$, or using the columns of $V$, it holds for the input mask that $\mathbf{C} = \sum_{i=1}^N {c}_i \, \mathbf{{v}}_i$ with $\mathbf{c}:= ({c}_1, \ldots, {c}_N)^\top$  the vector of coefficients. We now recall that by \eqref{eq:Gx_eigenbasis}
	\begin{equation*}
		G_{\mathbf{x}}=\sum_{j=0}^\infty A^j \mathbf{C} \mathbf{C}^\top (A^j)^\top
		=
		\sum_{i,j = 1}^N \varphi_{i,j}\, \mathbf{{v}}_i \, \mathbf{v}_j^\ast ,
	\end{equation*}	
	with $\varphi_{i,j}:=({c}_i \overline{c}_j)/(1 - \lambda_i \overline{\lambda}_j)$, and hence it holds that
	\begin{align*}
		%\label{VGVs}	
		V^{-1} G_{\mathbf{x}} (V^\ast)^{-1} 
		& =
		\left( \sum_{i,j = 1}^N \varphi_{i,j} \left( \mathbf{{v}}_k^\ast\, \mathbf{{v}}_i \, \mathbf{{v}}_j^\ast\, \mathbf{{v}}_l \right) \right)_{k,l}^N 
		= \left( \varphi_{k,l} \right)_{k,l}^N.
	\end{align*}
	
	Finally, using this expression in \eqref{eq:MC_tau_AC}, we can write $\textnormal{MC}_\tau$ as follows:
	\begin{align}
		\label{MC_neutr}
		\textnormal{MC}_\tau 
		& =
		\mathbf{C}^\top (A^\tau)^\top {G}_{\mathbf{x}}^{-1} A^\tau \mathbf{C}=\mathbf{C}^\top  (V^\ast)^{-1} (\Lambda^\ast)^\tau V^\ast  G_{\mathbf{x}}^{-1}  V \Lambda^\tau V^{-1}  \mathbf{C}\nonumber \\
		& =
		\mathbf{C}^\top (V^{-1})^\ast   (\Lambda^\ast)^\tau \left( \left( \varphi_{k,l} \right)_{k,l}^N \right)^{-1} \Lambda^\tau  V^{-1} \mathbf{C} 
		=
		\mathbf{c}^\ast  (\Lambda^\ast)^\tau \left( \left( \varphi_{k,l} \right)_{k,l}^N \right)^{-1} \Lambda^\tau \mathbf{c} \nonumber
		\\
		&= \mathbf{c}^\ast  (\Lambda^\ast)^\tau \left( \textnormal{diag}\left( \mathbf{c} \right) \left( \frac{1}{1 - \lambda_k \overline{\lambda}_l} \right)_{k,l}^N \textnormal{diag}\left( \mathbf{c}^\ast \right) \right)^{-1} \Lambda^\tau \mathbf{c}\nonumber
		\\
		& =
		\mathbf{c}^\ast  (\Lambda^\ast)^\tau \textnormal{diag}\left( \mathbf{c}^\ast \right)^{-1} \bigg( \bigg( \frac{1}{1 - \lambda_k \overline{\lambda}_l} \bigg)_{k,l}^N \bigg)^{-1} \textnormal{diag}\left( \mathbf{c} \right)^{-1} \Lambda^\tau \mathbf{c}		 \nonumber\\
		&= \boldsymbol{\iota}_N^\top\, (\Lambda^\ast)^\tau \bigg( \bigg( \frac{1}{1 - \lambda_k \overline{\lambda}_l} \bigg)_{k,l}^N \bigg)^{-1} \Lambda^\tau\, \boldsymbol{\iota}_N ,
	\end{align}
	where $\boldsymbol{\iota}_N = (1, \ldots, 1)^\top \in \mathbb{R}^N$. The last equality in the derivation follows from the commutative property of the product of diagonal matrices. 
	Hence, $\textnormal{MC}_\tau$ is independent of  $\mathbf{C}$ for all $\tau \in \mathbb{N}$ under the stated assumptions.
\end{proof}

A complementary result in continuous time with stationary inputs was derived by \cite{Hermans2010}. To the best of our knowledge, the previous proposition is the first derivation of this property in the context of discrete-time models. A generalization of the memory neutrality for weakly stationary inputs (possibly autocorrelated) is given in Theorem \ref{theorem:stationary_memory_neutral} in Appendix \ref{appendix_A}. 

\subsubsection{Another Formula for Memory Capacity}

The proof of Proposition~\ref{prop:input_mask_neutral} offers another additional strategy that one may follow in order to compute memory capacities. Indeed, the resulting closed-form expression \eqref{MC_neutr} can be used to evaluate the memory curve. More precisely, for a chosen reservoir matrix $A$ it is sufficient to compute its eigendecomposition $A = V \Lambda V^{-1}$, then construct the matrix
\begin{equation*}
	{L}_A := \left( \frac{1}{1 - \lambda_k \overline{\lambda}_l} \right)_{k,l}^N ,
\end{equation*}
and finally compute
\begin{equation*}
	\textnormal{MC}_\tau = \boldsymbol{\iota}_N^\top\, (\Lambda^\ast)^\tau {L}_A^{-1} \Lambda^\tau\, \boldsymbol{\iota}_N .
\end{equation*}
Unfortunately, similarly to all the previous approaches, this strategy still exploits the structure of the spectrum of $A$ and may suffer from the same ill-conditioning issues. Simple simulations, which, for the sake of brevity, we do not report, immediately show that regular matrix distributions produce ${L}_A$ matrices with eigenvalues decaying as quickly as those of the respective ${G}_{\mathbf{x}}$. This makes the direct application of Proposition~\ref{prop:input_mask_neutral} for memory evaluation also an infeasible option.

Despite the fact that the result of the neutrality of the LESN memory with respect to the choice of the input mask in Proposition~\ref{prop:input_mask_neutral} yields no immediate numerical advantages, it is nevertheless at the origin of robust numerical techniques for empirical memory evaluation that we present in the following sections. More explicitly, we shall show how to use the memory neutrality property to design a memory capacity estimation procedure that recovers full memory in linear ESN models and is robust with respect to the numerical issues discussed in Section~\ref{section:naive}.

\subsection{Krylov Conditioning}
\label{Krylov Conditioning}
In Section~\ref{Naive Algebraic Method} we showed that the normalized covariance matrices $G_{\mathbf{x}}$ intervene in the computation of capacities $\textnormal{MC}_{\tau}$, $\tau\in \mathbb{N}$. We now explain how one of the sources of numerical problems in memory capacity evaluation is due to the poor conditioning of Krylov matrices that are implicitly used in numerical procedures when evaluating $G_{\mathbf{x}}$.     

For $N\in \mathbb{N}$, $A \in \mathbb{M}_N$, and $\mathbf{C} \in \mathbb{R}^N$ define the Krylov matrix
\begin{equation*}
	K := \left( \mathbf{C} \,|\, A \mathbf{C} \,|\, A^2 \mathbf{C} \,|\, \ldots \right) ,
\end{equation*}
which is infinite in the column dimension. Under the hypothesis $\rho(A)<1 $, Gelfand's formula  \citep{lax:functional:analysis} guarantees that there exists $k _0 \in \mathbb{N} $ such that $\vertiii{A^{k _0}}_{\infty} <1 $ and hence for any $\epsilon>0 $ there exists $k \in \mathbb{N}  $ such that $\vertiii{A^{k}}_{\infty} <\epsilon $. We can use this fact to truncate the matrix $K$ to $m$ columns so that $\| A^m \mathbf{C}\|_{\infty} < {eps}$, with  $eps$ denoting the double-precision of floating numbers of the researcher's numerical software. 
%A lower bound for $m$ can be obtained as a function the spectral radius of $A$. 
Therefore, when using numerical tools, the finite-dimensional $m$-column Krylov matrix is used. We denote this matrix by
\begin{equation}\label{eq:krylov_mat}
	K_m := \left( \mathbf{C} \,|\, A \mathbf{C} \,|\, A^2 \mathbf{C} \,|\, \ldots \,|\, A^{m-1} \mathbf{C}\right)
\end{equation}
and notice that ${G}_{\mathbf{x}}$ can be approximated by the product of finite-dimensional matrices,
\begin{equation}
	\label{G_approx}
	\widetilde{G}_{\mathbf{x}} = K_m K_m^\top .
\end{equation}
A useful factorization of the finite-dimensional Krylov matrices is given by the following result, which we adapt from  Lemma 2.4 in \cite{meurantKrylovMethodsNonsymmetric2020} using our notation.

\begin{lemma}\label{lemma:krylov_vandermonde}
	Let $A \in \mathbb{M}_N$ be diagonalizable with $A = V \Lambda V^{-1}$, where matrix $\Lambda$ is diagonal, and let ${\mathbf{c}} = V^{-1} \mathbf{C}$. Then, the Krylov matrices  $K_m$ defined in \eqref{eq:krylov_mat} can be factorized as
	\begin{equation}
		K_m = V \, D_{{\mathbf{c}}} \, W_m ,
	\end{equation}
	where $D_{{\mathbf{c}}} = \textnormal{diag}( {\mathbf{c}} )$ and $W_{m} \in \mathbb{M}_{N,m}$ is the Vandermonde matrix 
	\begin{equation*}
		W_m := \left(\begin{array}{cccc}
			1 & \lambda_1 & \cdots & \lambda_1^{m-1} \\
			1 & \lambda_2 & \cdots & \lambda_1^{m-1} \\
			\vdots & \vdots & \ddots & \vdots \\
			1 & \lambda_N & \cdots & \lambda_N^{m-1}
		\end{array} \right)
	\end{equation*}
	constructed using the eigenvalues of $A$.
\end{lemma}

It is well-known that Krylov matrices are difficult to treat numerically. As pointed out in \cite{meurantKrylovMethodsNonsymmetric2020}, $K_m$ is often lacking in numerical rank when compared to the theoretical rank $N$, and, more importantly, it can have exponentially increasing conditioning number as $m$ grows. In our case, this phenomenon can be observed by noting that 
the eigenvalues of $\widetilde{G}_{\mathbf{x}}=K_m K_m^\top$ are the same as nonzero eigenvalues of $K_m^\top K_m$ for which it holds that

\begin{equation*}
	K_m^\top K_m =  W_m^\ast \, D_{{\mathbf{c}}}^\ast V^\ast V D_{{\mathbf{c}}} \, W_m.
\end{equation*}
The right-hand side of this expression, under the assumption that $A$ is normal and $D_{{\mathbf{c}}} = \mathbb{I}_N$, results in the positive-definite Hankel matrix $  W_m^\ast W_m$. \cite{tyrtyshnikovHowBadAre1994} proved that for real positive-definite Hankel matrices and general Krylov matrices the spectral condition number has exponential lower bounds in $m$, which means that $\widetilde{G}_{\mathbf{x}}$ can indeed be extremely ill-conditioned in many common setups.

\subsection{Memory Gaps and Krylov Subspace Squeezing}
As we already pointed out several times, Theorem 4.4 in 
\cite{RC15}  states that the total memory capacity \textnormal{MC} of a LESN equals
\begin{equation*}
	\textnormal{MC} = \text{rank}\{ K_N \},
\end{equation*}
with $K_N\in \mathbb{M}_N$ as in \eqref{eq:krylov_mat}. We refer to the discrepancy between this theoretical result and its numerical estimation as {\bfi memory gap}. The next paragraphs propose an explanation of why so often there is a disagreement between theoretical and empirically computed memory capacities. 

\subsubsection{Geometric Interpretation of Krylov Subspace Squeezing} 

We start by introducing the Krylov subspaces and their squeezing, which results in memory gaps in empirical exercises. We refer the reader to some interesting literature regarding the theory of Krylov subspace methods \citep{Bellalij2016}, its geometric aspects \citep{Eiermann2001}, and use for linear \citep{meurantKrylovMethodsNonsymmetric2020} and nonlinear systems \citep{Hashimoto2020}.

\begin{definition}[Krylov subspace]
	The {\bfi $\bm{j}$th-order Krylov subspace} generated by a matrix $A\in \mathbb{M}_N$ and a vector $\mathbf{C}\in \mathbb{R}^N$ is the linear subspace of $\mathbb{R}^N $ given by
	$$
	\mathcal{K}_j(A, \mathbf{C})=\operatorname{span}\left\{\mathbf{C}, A \mathbf{C}, A^2 \mathbf{C}, \ldots, A^{j-1} \mathbf{C}\right\}.
	$$	
\end{definition}

Let now $N$ be large and consider the $QR$ decomposition of the Krylov matrix 
\begin{equation*}
	K_N
	= \left( \mathbf{q}_1| \mathbf{q}_2| \ldots| \mathbf{q}_{N} \right) \left(\begin{array}{cccc}
		r_{1,1} & r_{1,2} & \ldots & r_{1,N} \\
		0 &  r_{2,2} & \ldots & r_{2,N} \\
		\vdots & \vdots& \ddots & \vdots \\
		0& 0& \ldots& r_{N,N}
	\end{array}\right)
	= Q R .
\end{equation*}
If $Q$ and $R$ are obtained via Gram-Schmidt orthogonalization, then the diagonal entries of $R$ have a clear geometric interpretation: each $r_{j,j}$ represents the norm of the orthogonal component in vector $A^j \mathbf{C}$ with respect to the subspace spanned by the columns of the Krylov matrix $K_{j}$, or equivalently $\mathcal{K}_{j}(A, \mathbf{C})$. Here, we ignore the fact that the Gram-Schmidt implementation of $QR$ is numerically unstable and instead focus on the fact that matrix $R$ inherits the rank structure of $K_N$. 

In practice, we observe that the size of $r_{j,j}$ decays \textit{superexponentially} compared to the decay of powers of $\rho(A)$, a phenomenon that we term {\bfi  Krylov subspace squeezing}. This means that for $A$ with $\rho(A)<1$ and with large enough $N$ there exists a positive integer $\ell < N$ such that numerically
\begin{equation*}
	R \approx \left(\begin{array}{cc}
		R_1 & R_2 \\
		\mathbb{O}_{N-\ell, \ell} & \mathbb{O}_{N-\ell, N-\ell}
	\end{array}\right)
	.
\end{equation*}
This implies that na\"ive methods, which do not control for the ill-conditioning of ${G}_{\mathbf{x}}$, lead to the incorrect estimate
\begin{equation*}
	\textnormal{MC} = \text{rank}\{ R \} \approx \ell .
\end{equation*}
We construct a simulation to showcase the Krylov subspace squeezing phenomenon for commonly chosen distributions of the reservoir connectivity matrix $A$. 

Let $j\in \mathbb{N}$ and denote as $\boldsymbol{\theta}_{j+1}={\rm perp}_{\mathcal{K}_{j}(A, \mathbf{C})}(A^j \mathbf{C}) \in  \mathcal{K}_{j}(A, \mathbf{C})^{\perp}$ the orthogonal component of $A^j \mathbf{C}$  with respect to  $\mathcal{K}_{j}(A, \mathbf{C})$, with $\|\boldsymbol{\theta}_{1}\| = \|\mathbf{C}\|$ and hence $\|{ \boldsymbol{\theta}_1} \|= 1$ due to normalization. To compute $\boldsymbol{\theta}_j$ in a robust fashion, we employ two different approaches. Firstly, one can use the {\bfi  Arnoldi iteration approach} \citep{arnoldi1951principle}, which is specially designed to handle the orthogonalization of Krylov iterations. Alternatively, one can define the projection $P_j^c: \mathbb{R}^N\longrightarrow \mathcal{K}_j(A, \mathbf{C})$, or, equivalently, $P_j^c: \mathbb{R}^N\longrightarrow \mathcal{C}({K}_j)$, with the corresponding projection matrix $P_j^c = K_j(K_j^\top K_j)^{-1}K_j^\top$. Additionally, we may take the singular value decomposition of $K_j$ given by
\begin{equation}\label{eq:krylov_svd}
	K_j = U_j \, \Sigma_j \, W_j^\top,
\end{equation}
where the columns of $U_j\in \mathbb{M}_{N,j}$ and $W_j\in \mathbb{M}_{j}$ are the orthonormal left-singular and right-singular vectors of $K_j$, respectively, and $\Sigma_j\in \mathbb{M}_{j}$ with $j$ singular values of $K_j$ on its diagonal, respectively. Hence, one obtains that the orthogonal components $\boldsymbol{\theta}_j$  for every $1 \leq j \leq N$ have the norm
\begin{align*}
	\| \boldsymbol{\theta}_{j+1} \| 
	&= \|{  \, (\mathbb{I}_N - P_j^c) A^j \mathbf{C}}\|=\|{  \, (\mathbb{I}_N - U_j \, \Sigma_j \, W_j^\top (W_j \Sigma_j U_j^\top U_j \, \Sigma_j \, W_j^\top)^{-1}W_j \Sigma_j U_j^\top) A^j \mathbf{C}}\|\\
	&=\|{  \, (\mathbb{I}_N - U_{j} U_{j}^\top) A^j \mathbf{C}}\|.
\end{align*}
We call this singular value decomposition approach the {\bfi orthogonal method}, as it explicitly removes dependence on the ill-conditioning that is now incorporated in the singular values matrices $\Sigma_j$.\footnote{An alternative option is, of course, to use a standard linear projection argument i.e. least-squares to compute $\boldsymbol{\theta}_j$. Due to the Krylov structure, however, this method is very ill-conditioned.} 

The results of our simulations with LESN models of size $N = 100$ are shown in Figure~\ref{fig:subs_krylov}, which also include a rank estimation of $K_N$. As one can notice, even with a logarithmic ordinate axis, the decay for most random sampling distributions is faster than exponential when compared to the powers of the leading eigenvalue. Only when using a random orthogonal matrix the decay of $\|{ \boldsymbol{\theta}_j} \|$ is close to $\rho(A)^{j-1}$, as shown in panel (d). 

Additionally, we make an important empirical observation: the value of $\|{ \boldsymbol{\theta}_j} \|$ as a function of $j$ is well approximated by the ordered cumulative product of the absolute values of eigenvalues of $A$. Our empirical finding can be seen in panels of Figure~\ref{fig:subs_krylov} by considering the dashed black line, which plots such cumulative eigenvalue product. This observation, combined with knowledge of the spectral properties of random matrices, allows getting a more precise understanding of the ill-conditioning of the reservoir autocovariance matrix, as we argue now.

\begin{figure}
	\centering
	\begin{subfigure}[b]{0.49\textwidth}
		\centering
		\includegraphics[width=\textwidth]{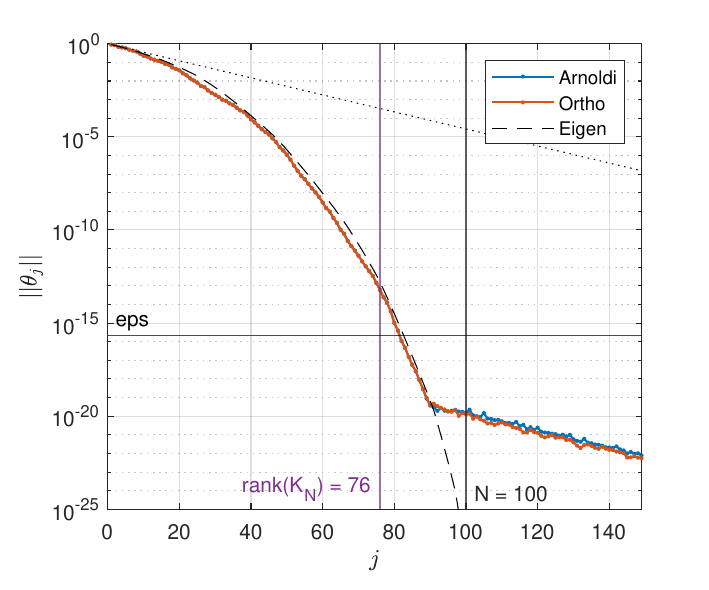}
		\caption{$A_{ij} \sim\ \text{i.i.d.}\ \mathcal{N}(0,1)$}
		\label{}
	\end{subfigure}
	%\hfill
	\begin{subfigure}[b]{0.49\textwidth}
		\centering
		\includegraphics[width=\textwidth]{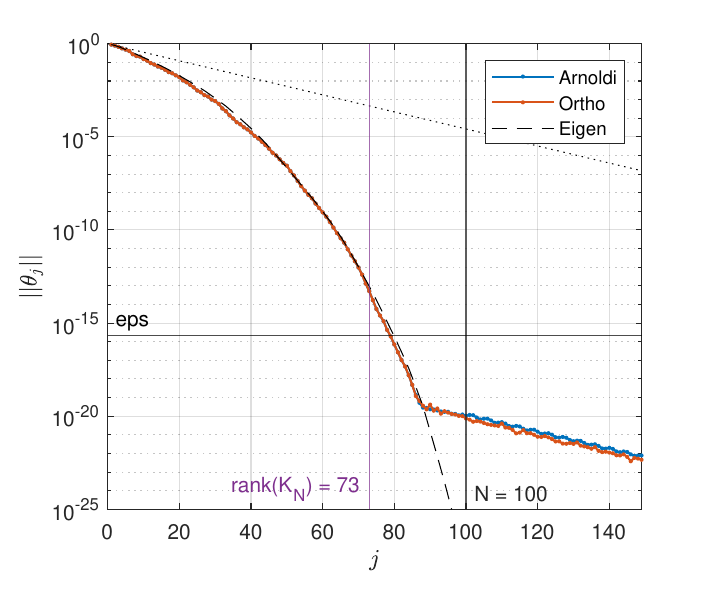}
		\caption{$A_{ij} \sim\ \text{i.i.d.}\ \mathcal{U}(-1,1)$}
		\label{}
	\end{subfigure}
	\\[15pt]
	%\hfill
	\begin{subfigure}[b]{0.49\textwidth}
		\centering
		\includegraphics[width=\textwidth]{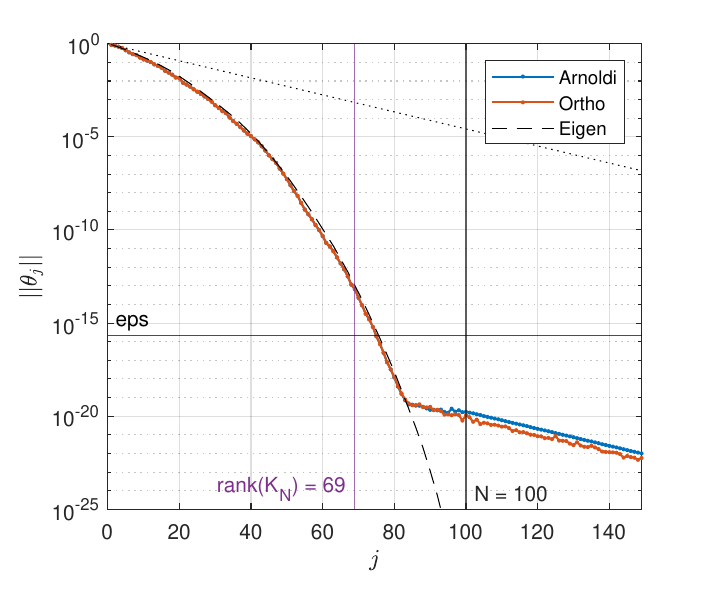}
		\caption{$A_{ij} \sim\ \text{i.i.d.}\ sp\mathcal{N}(0,1,0.1)$}
		\label{}
	\end{subfigure}
	%\hfill
	\begin{subfigure}[b]{0.49\textwidth}
		\centering
		\includegraphics[width=\textwidth]{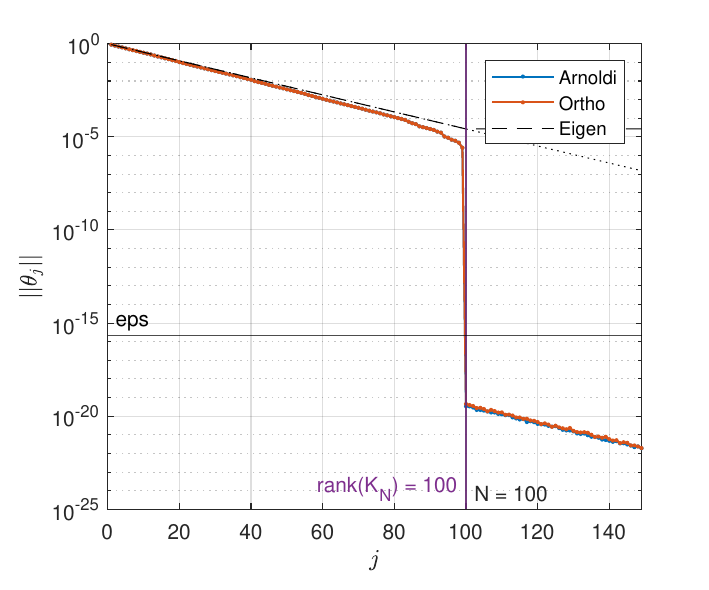}
		\caption{$A \sim\ \mathcal{O}(\mathcal{N}(0,1))$}
		\label{}
	\end{subfigure}
	\vspace{1em}
	\caption{Krylov subspace squeezing effects as measured using the norm of the orthogonal component for reservoir matrix $A = (A_{ij}) \in \mathbb{M}_N$, $\rho(A) = 0.9$, sampled $\mathcal{N}(0,1)$ in (a),  $\mathcal{U}(-1,1)$ in (b), sparse standard Gaussian with the degree of sparsity $0.1$, $sp\mathcal{N}(0,1,0.1)$, in (c), and orthogonal standard Gaussian in (d), and for Krylov matrix $K_m \in \mathbb{M}_{N, m}$, where in all plots $N = 100$ and $m = 5N$. Input mask is $\mathbf{C} = \boldsymbol{\iota}_N=(1, \ldots, 1)^\top \in \mathbb{R}^N$. The black \textit{dotted} line shows the exponential decay of leading eigenvalue $\rho(A)$, while the black \textit{dashed} line illustrates the approximate decay law derived using random matrix theory in Section~\ref{subsection:rmt_insights}. A solid black horizontal line shows the numerical double-precision of floating numbers in MATLAB, $eps =2^{-52}\approx 2.22\times 10^{-16}$.}
	\label{fig:subs_krylov}
\end{figure}

\subsubsection{Random Matrix Theory Insights}
\label{subsection:rmt_insights}

A fundamental result of random matrix theory (RMT) is the celebrated \textit{circular law}, which broadly speaking states that the (appropriately scaled) eigenvalues of families of random matrices are asymptotically uniformly distributed on the complex unit circle as $N \to \infty$ (see \citealt{taoTopicsRandomMatrix2012a} for an introductory discussion). A general statement of the circular law for ensembles of matrices with i.i.d.~entries with unit variance was given by \cite{taoRandomMatricesUniversality2010}. Matrix ensembles with {sparse} entries also obey the circular law, as proven by \cite{woodUniversalityCircularLaw2012} and \cite{basakCircularLawSparse2019a}. In particular, the degree of sparsity controls the probability of singularity, although in appropriate settings such probability remains exponentially small \citep{basakInvertibilitySparseNonHermitian2017}. Figure~\ref{fig:A_eigvals} in the Appendix shows the distribution of eigenvalues for commonly used ESN reservoir matrices. Notice that for ensembles of Gaussian, sparse Gaussian, and uniform entries, the associated eigenvalues have a close-to-uniform distribution on the complex unit circle. We highlight that attempts to use RMT to gain insights on reservoir models have already been made: \cite{zhang:echo} use the circular law to derive explicit bounds on spectral scaling factors; \cite{couillet2016Proc} and \cite{linearESN} apply results from random matrix theory to make performance analyses of linear echo state networks effected by exogenous noise. However, to the best of our knowledge, our empirical observations are new.

With the circular law in mind, in linear reservoirs where $A$ is drawn randomly from standard matrix ensembles, we know that for its eigenvalues it approximately holds that $|\lambda_i|^2 \sim \mathcal{U}(0, \rho(A))$, $i\in \{1,\ldots, N\}$. We can thus derive the following closed-form approximation, call it $\kappa_j$, for the value of $\| \boldsymbol{\theta}_j \|$, given by $\kappa_1=1$ and 
\begin{equation*}
	\kappa_{j+1} = \sqrt{\rho(A) \, \frac{N !}{N^{j} (N-j) !}} \,,  \enspace \forall j\geq 1.
\end{equation*}
This expression can be derived easily by noting that $|\lambda_i|$ are approximately distributed as $\sqrt{\rho(A) Z_i}$ where $Z_i \sim \mathcal{U}(0,1)$. When $N$ is large, we may further approximate realizations $( Z_i )_{i=1}^N$ with a uniform grid of knots over $(0, 1)$. Computing the cumulative product of these knots in descending order gives the formula for $\kappa_j$ above. It shows that the decay of $\|{ \boldsymbol{\theta}_j} \|$ can be indeed much faster than that of powers of $\rho(A)$ under the circular law. In theory, a sharper formula could be derived by noting that most eigenvalues of a random matrix come in conjugate complex pairs, so knots should also be chosen in couples. This would require knowing the expected ratio of real to complex eigenvalues of a random matrix ensemble, which is beyond the scope of this discussion. 
Yet, as shown with the dashed black lines in Figure~\ref{fig:subs_krylov}, we empirically find that our RMT approximation is remarkably precise at fitting the faster-than-exponential decay of $\|{ \boldsymbol{\theta}_j} \|$.

\subsection{Subspace Methods}

Now that we have discussed the ill-posed nature of inverse problems involving ${G}_{\mathbf{x}}$ and how it relates to its Krylov structure, we propose two approaches that can correctly recover the full memory capacity as well as individual lag-$\tau$ capacities. Our methods boil down to the idea of using appropriate matrix decompositions to remove those parts of the singular values spectrum from normalized reservoir autocovariance ${G}_{\mathbf{x}}$ that lead to its ill-conditioning. 

\subsubsection{Orthogonalized Subspace Method}
We start by recalling again the expression of the $\tau$-lag memory capacity of the LESN given in \eqref{eq:MC_tau_AC}, namely
\begin{equation}
	\label{eq:MC_tau_AC2}
	\textnormal{MC}_\tau =\mathbf{C}^\top (A^\tau)^\top {G}_{\mathbf{x}}^{-1} A^\tau \mathbf{C},
\end{equation}
where, as we explained  in Subsection~\ref{Krylov Conditioning}, ${G}_{\mathbf{x}}$ can be approximated by
\begin{equation}
	\widetilde{G}_{\mathbf{x}} = K_m K_m^\top
\end{equation}
with $K_m\in \mathbb{M}_{N,m}$ a Krylov matrix with the column dimension truncated up to $m$ as in \eqref{eq:krylov_mat}. In this case, the approximate memory capacities in \eqref{eq:MC_tau_AC2} can be computed as the diagonal elements of the following matrix:
\begin{equation}
	\label{Pmr}
	P_m^r := K_m^\top (K_m K_m^\top )^{-1} K_m.
\end{equation}
It is easy to see that this is a projection matrix corresponding to the projection operator $P_m^r:\mathbb{R}^m\longrightarrow \mathcal{C}(K_m^\top)\subset  \mathbb{R}^m$.
Using the singular value decomposition $K_m = U_{m} \, \Sigma_{m} \, W_{m}^\top$, with $U_{m} \in \mathbb{M}_{N}$ full-rank orthogonal with the left-singular vectors of $K_m$ as columns, $\Sigma_{m} \in \mathbb{M}_{N}$ diagonal, and $W_{m} \in \mathbb{M}_{m,N}$  with the right-singular orthonormal vectors of $K_m$ as  columns (notice that this SVD is different from the one used in \eqref{eq:krylov_svd}), we write \eqref{Pmr} as
\begin{align*}
	P_m^r  
	& = W_{m} \Sigma_{m}  U_{m}^\top ( U_{m} \Sigma_{m} W_{m}^\top W_{m} \Sigma_{m}  U_{m}^\top )^{-1} U_{m} \Sigma W_{m}^\top  =  W_{m} W_{m}^\top .
\end{align*}
Therefore each $\tau$-lag memory capacity of the LESN, $1\leq \tau\leq m$, is well approximated by $(P_m^r)_{\tau,\tau}$.

It is important to underline that this method of memory capacity computation does not suffer from any of the previously mentioned matrix or linear system inversion issues, as it sidesteps the computation of ${G}_{\mathbf{x}}^{-1}$ altogether. The core idea is to explicitly exploit the subspace structure of the Krylov matrix $K_m$ and, by using the singular value decomposition, to extract the projection matrix associated with the LESN memory capacity. We term this approach the {\bfi  orthogonalized subspace method} (OSM) and define
\begin{equation}\label{eq:MC_subsp}
	\text{MC}_\tau^{\rm OSM} = ( W_{m} W_{m}^\top )_{\tau,\tau}.
\end{equation}
Figures \ref{fig:subspace_methods_compare}-\ref{fig:subspace_methods_compare_more} show that the {orthogonalized subspace} method computes memory curves consistent with full memory. One also notices one of the downsides of this method when inspecting the memory curves recovered by OSM in all the panels of Figure~\ref{fig:subspace_methods_compare} and in subfigure (a) and (b) of Figure~\ref{fig:subspace_methods_compare_more}. More precisely, OSM results in memory capacity curves that need not be monotonically decreasing. This is in contrast to the known monotonicity of memory proven in \cite{Jaeger:2002}. A reason for this is that, while the subspace methods avoid a costly and unstable matrix inversion, it still relies on the computation of singular value decomposition factors. Formula (\ref{eq:MC_subsp}) does not guarantee that the diagonal entries are numerically non-increasing. Monotonicity hinges on recursively identifying the largest leading singular direction at each step of the decomposition. Accordingly, the accuracy of the estimated $\text{MC}_\tau$ is tied to the accuracy of the singular value decomposition, and the ill-conditioning of $K_m$ still plays some role in it.

\subsubsection{Averaged Orthogonalized Subspace Method}

Finally, we propose an improved version of our subspace memory computation method that exploits the input mask memory neutrality result established in Proposition~\ref{prop:input_mask_neutral}. Our goal is to leverage this property to produce a better approximation of $\text{MC}_\tau$ that is also monotonic. 

We first notice that even though the expression of the true memory capacity $\text{MC}_\tau$ in \eqref{MC_neutr} does not depend on $\mathbf{C}$, its numerical computation with OSM in \eqref{eq:MC_subsp} is impacted by the input mask. To recover the true memory capacity out of \eqref{eq:MC_subsp}, one is ultimately interested in computing $\mathbb{E}_{\mathbf{C}}[( W_{m} W_{m}^\top )_{\tau,\tau}]$, $1\leq \tau \leq m$, $m\in \mathbb{N}$, which, by Proposition~\ref{prop:input_mask_neutral}, should not depend on a particular choice of the distribution $p_{\mathbf{C}}$ of the input mask. Although one can potentially choose any $p_{\mathbf{C}}$ that would allow one to evaluate this integral, we do not find obtaining its expression in a closed form feasible. Our proposal is to adhere to the sample estimator or the Monte Carlo estimator of $\mathbb{E}_{\mathbf{C}}[( W_{m} W_{m}^\top )_{\tau,\tau}]$, $1\leq \tau \leq m$, $m\in \mathbb{N}$ which we will call the {\bfi  averaged orthogonalized subspace method}, or simply {OSM+}.

More explicitly, consider a sample of  $L$ independent and identically distributed according to some arbitrary chosen law $p_{\mathbf{C}}$ input masks $\{\mathbf{C}^{(1)}, \ldots, \mathbf{C}^{(L)}\}$, and  using \eqref{eq:MC_subsp} construct the following memory capacity curve estimator:  
\begin{align}
	\text{MC}_{L,\tau}^{\rm OSM+} = \frac{1}{L} \sum_{\ell=1}^L \left( W_m^{(\ell)}{ W_m^{(\ell)}}^\top \right)_{\tau,\tau}, \enspace 1\leq \tau \leq m,
\end{align}
for which the weak law of large numbers implies that
\begin{align*}
	\text{MC}_{L,\tau}^{\rm OSM+} 
	\xrightarrow[L \rightarrow \infty]{ p}\,
	\mathbb{E}_{\mathbf{C}}[( W_{m} W_{m}^\top )_{\tau,\tau}], \enspace 1\leq \tau \leq m.
\end{align*}
As mentioned above, one of the key advantages of this construction is the fact that the OSM+ method allows choosing any type of $p_{\mathbf{C}}$ as long as the conditions of Proposition~\ref{prop:input_mask_neutral} are satisfied.
Moreover, OSM+ is straightforward to implement numerically, as shown by the pseudo-code in Algorithm~\ref{algorithm:osm_plus}. Note that construction of the Krylov matrix can be done iteratively, and therefore the most computationally expensive operation is the singular value decomposition.
Figures~\ref{fig:subspace_methods_compare} and \ref{fig:subspace_methods_compare_more} show that the averaged subspace memory curves produced by OSM+ are indeed monotonic, in contrast to the one-step subspace approximation produced by OSM. Here, we do not make any suggestion as for the choice of the distribution for the entries of $\mathbf{C}$ since Figure~\ref{fig:subspace_methods_compare} indicates that common choices yield very similar results for moderate resampling size $L = 1000$. 

We emphasize that the na\"ive memory capacity computation discussed in Section~\ref{Naive Algebraic Method} is able to recover full memory {\it only} for some particular choices of the connectivity architectures for which, by construction,  the ill-conditioning problem is not pronounced. Indeed, the subplots (c) and (d) in Figure~\ref{fig:subspace_methods_compare_more} indicate that all three methods, namely na\"ive, OSM, and OSM+, in these cases correctly quantify full memory of linear recurrent networks.

\begin{algorithm}[t]
	\caption{Averaged Orthogonalized Subspace Method (OSM+)}\label{algorithm:osm_plus}
	\SetKwComment{Comment}{\#~}{} 
	\SetKwFunction{zeros}{zeros}
	\SetKwFunction{zeros}{zeros}
	\SetKwFunction{samplerandmatrix}{sample\_rand\_matrix}
	\SetKwFunction{svd}{svd}
	\SetKwFunction{diag}{diag}
	\SetKwInOut{Input}{Input}
	\SetKwInOut{Output}{Output}
	\Input{Reservoir connectivity matrix $A \in \M_{N}$, distribution $p_{\mathbf{C}}$ of input matrix $\mathbf{C}$, Krylov matrix truncation order $m \in \N$, sampling budget $L \in \N$}
	\Output{Memory capacity curve $\text{MC}_\tau$ for $0 \leq \tau \leq m$}
	\BlankLine
	\texttt{mc\_curve} $\leftarrow$ \zeros{m+1, 1}\Comment*{initialize MC curve vector}  
	\For{$\ell \leftarrow 1$ \KwTo $L$}{
		$\mathbf{C}^{(\ell)} \leftarrow$ \samplerandmatrix{$p_{\mathbf{C}}$; N, 1}\Comment*{sample input matrix}
		$K_{m}^{(\ell)} \leftarrow (\, \mathbf{C}^{(\ell)} \,\vert\, A \mathbf{C}^{(\ell)} \,\vert\, A^2 \mathbf{C}^{(\ell)} \,\vert\, \ldots \,\vert\, A^{m-1} \mathbf{C}^{(\ell)} \,)$\Comment*{construct Krylov matrix}
		$U_m^{(\ell)}$\texttt{,}$\Sigma_m^{(\ell)}$\texttt{,}${W_m^{(\ell)}}^\top \leftarrow$ \svd{$K_{m}^{(\ell)}$}\;
		\texttt{mc\_curve} $\leftarrow$ \texttt{mc\_curve} + $L^{-1}$\diag{${W_m^{(\ell)}} {W_m^{(\ell)}}^\top$}\Comment*{update estimate}
	}
\end{algorithm}

\begin{figure}
	\centering
	\begin{subfigure}[b]{0.49\textwidth}
		\centering
		\includegraphics[width=\textwidth]{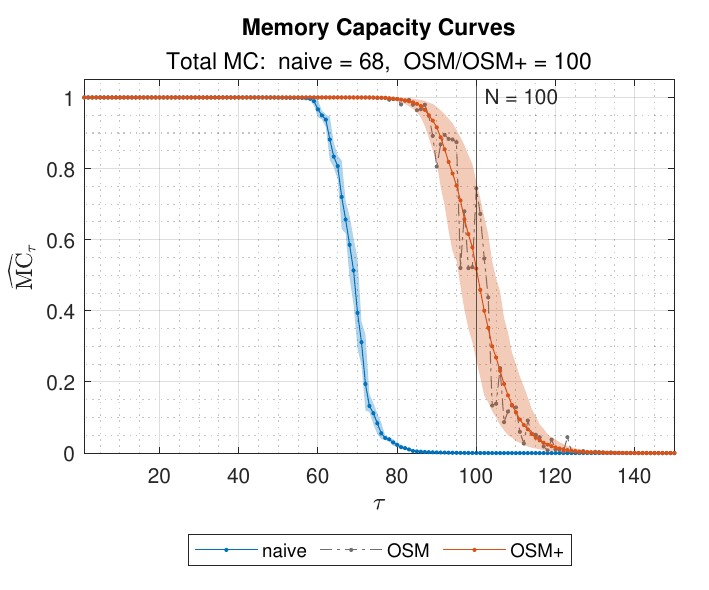}
		\caption{$\mathbf{C} \sim \text{i.i.d.}\ \mathcal{N}(0,1)$}
		\label{}
	\end{subfigure}
	%\hfill
	\begin{subfigure}[b]{0.49\textwidth}
		\centering
		\includegraphics[width=\textwidth]{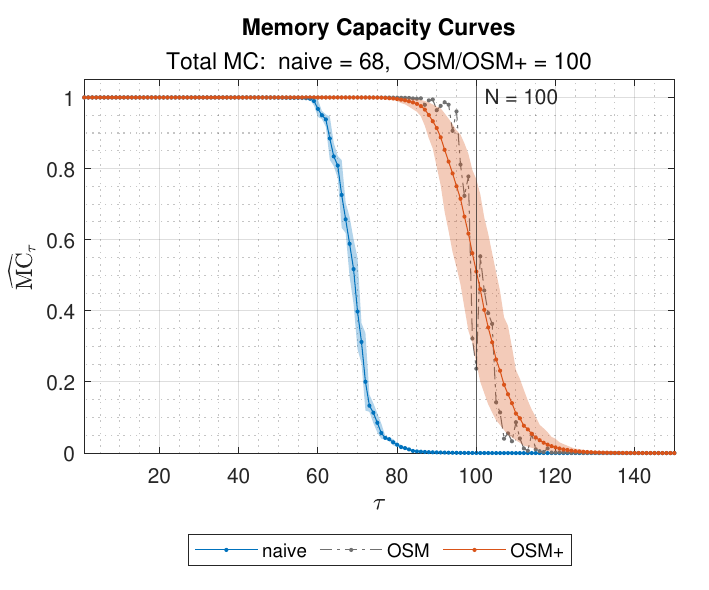}
		\caption{$\mathbf{C} \sim \text{i.i.d.}\ \mathcal{U}(-1,1)$}
		\label{}
	\end{subfigure}
	\\[15pt]
	\begin{subfigure}[b]{0.49\textwidth}
		\centering
		\includegraphics[width=\textwidth]{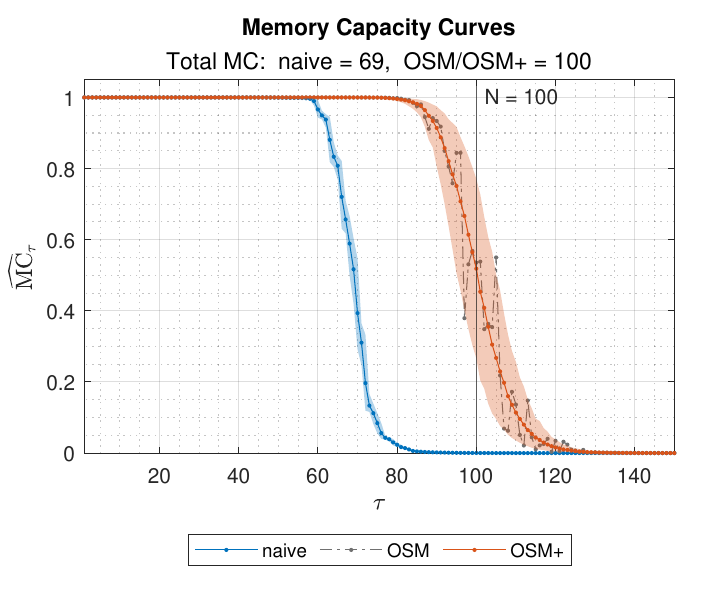}
		\caption{$\mathbf{C} \sim \text{i.i.d.}\ sp\mathcal{N}(0,1,0.1)$}
		\label{}
	\end{subfigure}
	%\hfill
	\begin{subfigure}[b]{0.49\textwidth}
		\centering
		\includegraphics[width=\textwidth]{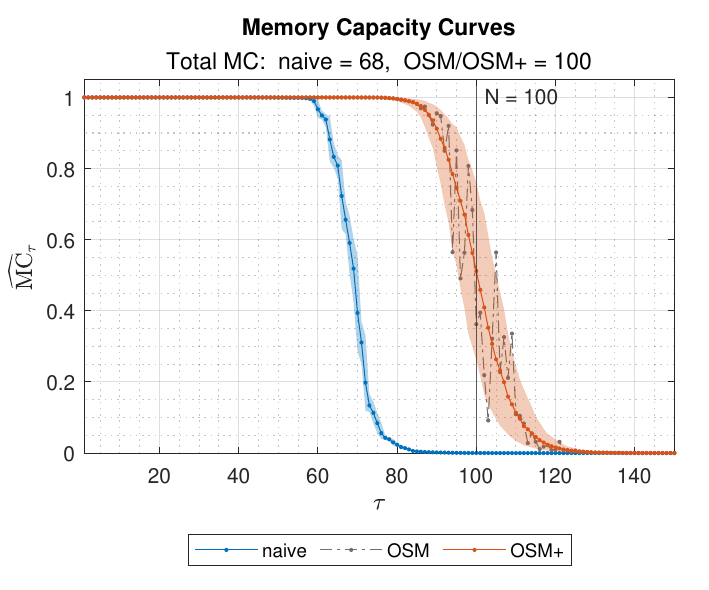}
		\caption{$\mathbf{C} \sim \text{i.i.d.}\ sp\,\mathcal{U}(0,1,0.1)$}
		\label{}
	\end{subfigure}
	\vspace{1em}
	\caption{Memory capacity curves of LESNs with connectivity matrix $A = (A_{ij}) \in \mathbb{M}_N$ with $\rho(A) = 0.9$. In all panels $A_{i,j}$ are sampled as i.i.d. degree $0.1$ sparse standard normal, $\ sp\mathcal{N}(0,1,0.1)$, and the input mask $\mathbf{C} = (c_{i}) \in \mathbb{R}^N$ is sampled as $\mathcal{N}(0,1)$ in (a), $\mathcal{U}(-1,1)$ in (b), degree $0.1$ sparse Gaussian, $sp\mathcal{N}(0,1,0.1)$, in (c), and degree $0.1$ sparse uniform, $sp\,\mathcal{U}(0,1,0.1)$, in (d). $\mathbf{C}$ is normalized after sampling to have a unit norm. Total MC is estimated as the sum of $\text{MC}_\tau$'s up to $1.5 N$ terms. For OSM+ the input mask $\mathbf{C}$ is resampled $L = 1000$ times to compute the average memory curve (lines) and $90\%$ frequency bands for $\text{MC}_\tau$ (shaded).  }
	\label{fig:subspace_methods_compare}
\end{figure}

\begin{figure}
	\centering
	\begin{subfigure}[b]{0.49\textwidth}
		\centering
		\includegraphics[width=\textwidth]{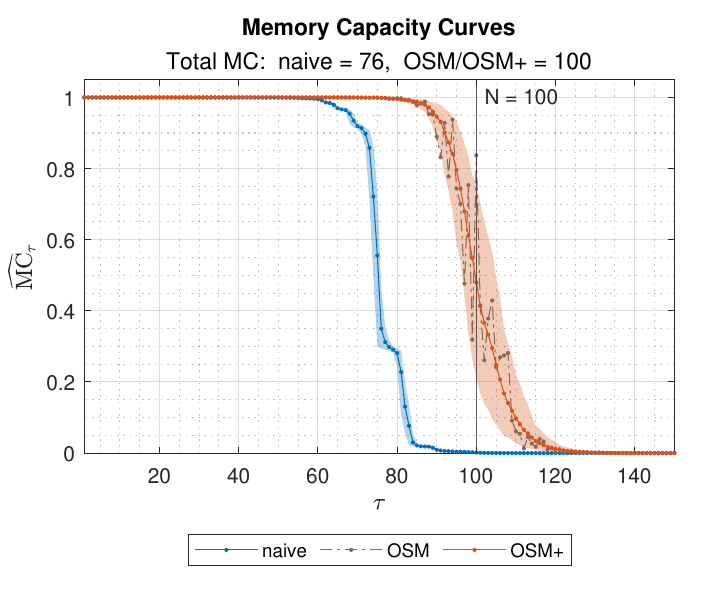}
		\caption{${A} \sim \text{i.i.d.}\ \mathcal{N}(0,1)$}
		\label{}
	\end{subfigure}
	%\hfill
	\begin{subfigure}[b]{0.49\textwidth}
		\centering
		\includegraphics[width=\textwidth]{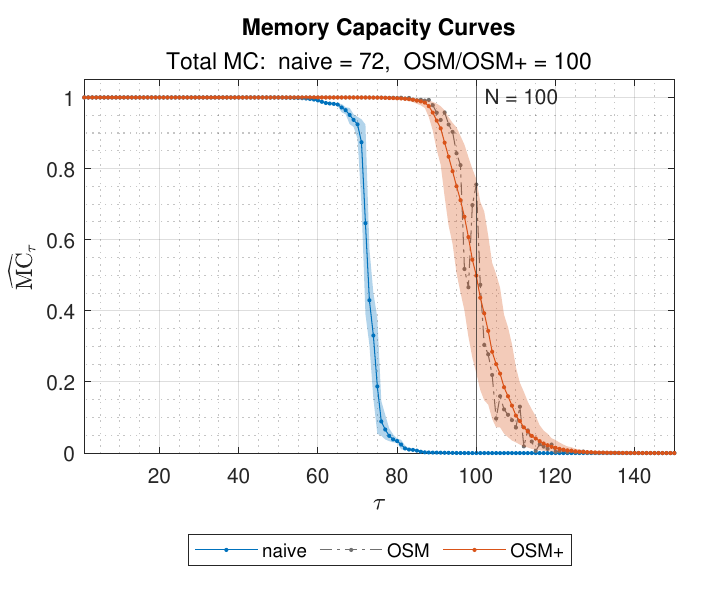}
		\caption{${A} \sim \text{i.i.d.}\ \mathcal{U}(-1,1)$}
		\label{}
	\end{subfigure}
	\\[15pt]
	\begin{subfigure}[b]{0.49\textwidth}
		\centering
		\includegraphics[width=\textwidth]{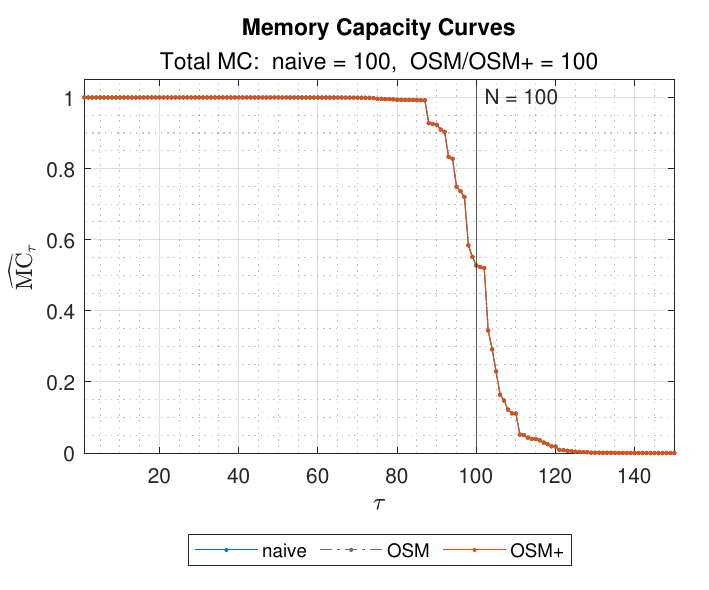}
		\caption{${A} \sim\ \mathcal{O}(\mathcal{N}(0,1))$}
		\label{}
	\end{subfigure}
	%\hfill
	\begin{subfigure}[b]{0.49\textwidth}
		\centering
		\includegraphics[width=\textwidth]{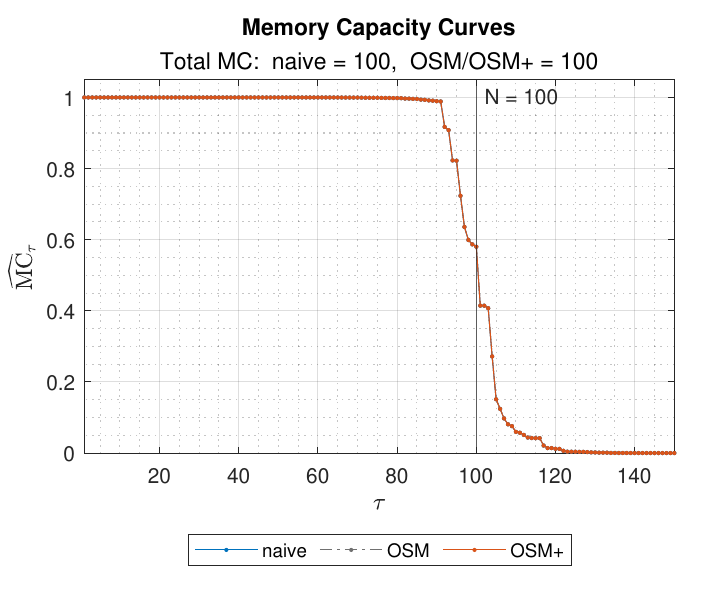}
		\caption{${A} \sim\ sp_C\mathcal{N}(0,1,0.1,0.7)$}
		\label{}
	\end{subfigure}
	\vspace{1em}
	\caption{Memory capacity curves of LESNs with input mask $\mathbf{C} = (c_{i}) \in \mathbb{R}^N$ and connectivity matrix $A = (A_{ij}) \in \mathbb{M}_N$, $\rho(A) = 0.9$, sampled from different standard distributions (in panel (d) $sp_C\mathcal{N}(0,1,0.1,0.7)$ stands for sparse standard Gaussian with sparsity degree $0.1$ and condition number $0.7$). In all panels $c_{i} \sim\ \text{i.i.d.}\ sp\mathcal{N}(0,1,0.1)$. $\mathbf{C}$ is normalized after sampling to have a unit norm. Total MC is computed as the sum of $\text{MC}_\tau$'s up to $1.5N$ terms. For OSM+ the input mask $\mathbf{C}$ is resampled $L = 1000$ times to compute the average memory curve (lines) and $90\%$ frequency bands for $\text{MC}_\tau$ (shaded).
	}
	\label{fig:subspace_methods_compare_more}
\end{figure}

\section{Discussion}
\label{section:discussion}

Given the fact that in this paper we have recalled (Section~\ref{section:naive}) and newly introduced (Section~\ref{section:subspace}) many methods to compute the memory capacity of linear recurrent networks, specifically focusing on LESN models, we now wish to provide an overview of the key insights we have gathered by comparing them.

Simulations and na\"ive algebraic methods are both plagued by significant issues. In the former case, estimating moments of a stochastic process with simulated data always introduces some positive bias in the calculation of $\textnormal{MC}_\tau$, yielding memory curves that are inconsistent with the theoretical properties of memory. In the latter case, na\"ive algebraic applications of close-form formulas for memory eventually resort to inverting generally ill-poised covariance matrices, see \eqref{eq:Gx_series}--\eqref{eq:Gx_eigenbasis} and Figure~\ref{fig:G_x_eigenvalues}. The numerical instability in the inversion of $G_\mathbf{x}$ is the core issue with these approaches, and it is unavoidable by all techniques that directly rely on expression \eqref{eq:MC_tau_covvar}. Indeed, this means that simulation methods, too, are eventually impacted, as in extremely large simulations the conditioning of $\widehat{\textnormal{Var}}(\mathbf{x}_t)$ is close to that of $\Gamma_{\mathbf{x}}$, and thus, ultimately, $G_{\mathbf{x}}$.

While our OSM and, especially, OSM+ proposals are theoretically grounded and numerically well-conditioned approaches to estimating $\textnormal{MC}_\tau$, it is important also to mention that in this paper we do not provide theoretical results on the convergence properties of these algorithms. From a practical perspective, it would be interesting to derive a rate of convergence for subspace methods as $m \to \infty$ and $\tau$ is fixed or $\tau \to \infty$. The former seems easier, while the latter seems useful in providing a better understanding of how memory behaves at the ``state dimension boundary'', $\tau \approx N$, as $N \to \infty$, too. We leave these developments to future work.

Lastly, we mention how our results may be generalized to \textbf{\textit{forecasting capacity}}. Following again \cite{RC15}, recall that the forecasting capacity of an ESN is given by
\begin{equation}\label{eq:FC_tau_covvar}
	\text{FC}_\tau =
	\frac{\text{Cov}(z_t, \mathbf{x}_{t-\tau}) {\Gamma}_{\mathbf{x}}^{-1} \text{Cov}(\mathbf{x}_{t-\tau}, z_t)}
	{\text{Var}(z_t)}, \enspace \tau\in \mathbb{N}_+ ,
\end{equation}
where the states $\mathbf{x}_{t-\tau}$ are now lagged and not the inputs. Then, $\text{FC}_\tau$ is a linear measure of the predictability of future inputs with respect to the currently available state. Following our discussion above on the implications of relying on \eqref{eq:MC_tau_covvar}, computations of $\text{FC}_\tau$ based on either simulations or na\"ive algebraic derivations are bound to provide inherently poor results. Although it is, in principle, possible to extend our OSM(+) methods also to compute $\text{FC}_\tau$ in a robust fashion, Corollary~3.5 in \cite{RC15} proves that, for generic ESN models (not necessarily linear), if $(z_t)_{t \in \Z_-}$ is a sequence of independent random variables, then $\text{FC}_\tau = 0$ and thus $\textnormal{FC} := \sum_{\tau = 0}^\infty \text{FC}_\tau = 0$. Therefore, for other types of stochastic inputs, forecasting capacity should be analyzed not as an inherent property of the (L)ESN model, but rather as a quantity based on the interaction between the model and inputs.

\section{Conclusion}
\label{section:conclusion}

In this paper, we have provided an overview of the existing literature on memory capacity measures for recurrent neural networks and the approaches that have been extensively used in designing memory-optimal network architectures. 

We have focused on explaining and providing solutions for what we call the linear memory gap, which refers to the difference between empirically measured memory capacities and their provable theoretical values. We have demonstrated that this discrepancy arises due to numerical artifacts that have been overlooked in previous studies. 

We propose robust techniques for the accurate estimation of memory capacity, which result in full memory results for linear RNNs, as should be generically expected. Our findings suggest that previous efforts to optimize memory capacity for linear recurrent networks may have been plagued with numerical artifacts, leading to incorrect results. We base our findings on the fact that the capacities of linear systems are generically full, disregarding the particular choice of architecture. We also show that the memory capacity is neutral to the choice of the input mask. We propose two orthogonalized subspace methods that allow empirically recovering the full memory of linear systems and render results consistent with the theory. 

We hope, with this conclusive work, to close the door to forthcoming attempts at memory optimization for linear RNNs that are not justified from a theoretical point of view.

\vspace{2em} 

\subsection*{Acknowledgements}

GB and JPO thank the hospitality and the generosity of the University of St. Gallen, where part of this work was completed. LG thanks the Nanyang Technological University for the hospitality, which made some of this work possible.
GB thanks Konstantin Usevich for the helpful discussion of algebraic insights regarding the empirical conjecture on ordered eigenvalue products.	
JPO acknowledges financial support from the Nanyang Technological University (grant number 020870-00001) and the Swiss National Science Foundation (grant number 200021\_175801/1).

\newpage

%% --- APPENDICES ---------------------------------------

\newpage

\appendix

\section{Memory Neutrality to Input Mask Under Stationarity}
\label{appendix_A}
We now show that we can generalize Proposition~\ref{prop:input_mask_neutral} to the case of weakly stationary input processes that are not necessarily white noises. This provides a discrete-time counterpart to the result in \cite{Hermans2010} and allows us to apply our memory estimation methods in more general setups in which just the stationarity of the input is needed.

\begin{theorem}
	\label{theorem:stationary_memory_neutral}
	Under the controllability assumptions of Proposition~\ref{prop:MC}, for any weakly stationary input $(z_t)_{t \in \mathbb{Z}}$ (not necessarily white noise), the memory of a linear echo state network is neutral to the choice of the input mask $\mathbf{C}$.
\end{theorem}

\begin{proof}
	We shall mimic the proof of Proposition~\ref{prop:input_mask_neutral}. We start by noticing that under the assumptions of the theorem, the stationarity of the input process implies stationarity of the associated states process $(\mathbf{x}_t)_{t\in \mathbb{Z}}$ as well as of the joint process $(z_{t+\tau}, \mathbf{x}_t)_{t \in \mathbb{Z}}$ for any $\tau \in \mathbb{N}$ (see Corollary 2.4 in \citealt{RC15}), and calculate
	\begin{align*}
		\text{Cov}(\mathbf{x}_{t}, {z}_{t+\tau}) 
		& =
		\text{Cov}(\mathbf{x}_{0}, {z}_{\tau}) 
		=
		\mathbb{E}\left[  \sum_{j=0}^\infty A^j \mathbf{C} {z}_{-j}  {z}_\tau \right] 
		= 
		\sum_{k=1}^N  \sum_{j=0}^\infty \lambda_k^j \mathbb{E}\left[{z}_{-j} {z}_\tau \right]  {c}_k \mathbf{{v}}_k  \\
		& = 
		\sum_{k=1}^N \left( \sum_{j=0}^\infty \lambda_k^j \gamma(\tau - j) \right) {c}_k \mathbf{{v}}_k  = 
		\sum_{k=1}^N g_k(\tau) {c}_k \mathbf{{v}}_k,
	\end{align*}
	The state autocovariance matrix is given by
	\begin{align*}
		\Gamma_{\mathbf{x}} 
		& =
		\mathbb{E}\left[ \left( \sum_{i=0}^\infty A^i \mathbf{C} z_{-i} \right) \left( \sum_{j=0}^\infty A^j \mathbf{C} z_{-j} \right)^\top \right] \\
		& =
		\mathbb{E}\left[ \sum_{i=0}^\infty A^i \mathbf{C} \mathbf{C}^\top (A^\top)^i z_{-i}^2 + \sum_{j \geq 1} \sum_{i=0}^\infty \left\{ A^{i} \mathbf{C} \mathbf{C}^\top (A^\top)^{i+j} + A^{i+j} \mathbf{C} \mathbf{C}^\top (A^\top)^{i} \right\} z_{-i} z_{-i-j}  \right] \\
		& =
		\sum_{i=0}^\infty A^i \mathbf{C} \mathbf{C}^\top (A^\top)^i \gamma(0) + \sum_{j \geq 1} \sum_{i=0}^\infty \left\{ A^{i} \mathbf{C} \mathbf{C}^\top (A^\top)^{i+j} + A^{i+j} \mathbf{C} \mathbf{C}^\top (A^\top)^{i} \right\} \gamma(j) .
	\end{align*}
	We now analyze all three summands separately:
	\begin{align*}
		\sum_{i=0}^\infty A^i \mathbf{C} \mathbf{C}^\top (A^\top)^i \gamma(0) &=
		\gamma(0)  \sum_{k,l=1}^N {c}_k \overline{{c}}_l \frac{1}{1 - \lambda_k \overline{\lambda}_l}  \mathbf{{v}}_k {\mathbf{{v}}}^\ast_l ,\\
		\sum_{i=0}^\infty A^{i} \mathbf{C} \mathbf{C}^\top (A^\top)^{i+j} \gamma(j) 
		& =
		\gamma(j) \sum_{k,l=1}^N {c}_k \overline{c}_l  \frac{\overline{\lambda}_l^j }{1 - \lambda_k \overline{\lambda}_l}  \mathbf{{v}}_k \mathbf{{v}}_l^\ast ,\\
		\sum_{i=0}^\infty A^{i+j} \mathbf{C} \mathbf{C}^\top (A^\top)^{i} \gamma(j)
		& =
		\gamma(j) \sum_{k,l=1}^N {c}_k \overline{c}_l  \frac{\lambda_k^i }{1 - \lambda_k \overline{\lambda}_l}  \mathbf{{v}}_k \mathbf{{v}}_l^\ast,
	\end{align*}
	and can simplify $\Gamma_{\mathbf{x}}$ as 
	\begin{align*}
		\Gamma_{\mathbf{x}} 
		& =
		\sum_{k,l=1}^N {c}_k \overline{c}_l \left( \frac{\gamma(0)}{1 - \lambda_k \overline{\lambda}_l} + \sum_{j \geq 1} \gamma(j) \frac{\lambda_k^j + \overline{\lambda}_l^j}{1 - \lambda_k \overline{\lambda}_l} \right) \mathbf{{v}}_k \mathbf{{v}}_l^\ast  = 
		\sum_{k,l=1}^N {c}_k \overline{c}_l \frac{h_{k,l}}{1 - \lambda_k \overline{\lambda}_l} \mathbf{{v}}_k \mathbf{{v}}_l^\ast .
	\end{align*}
	We can now mimic the derivation in the proof of Proposition~\ref{prop:input_mask_neutral} as follows
	\begin{align*}
		\textnormal{MC}_\tau 
		& =
		\gamma(0)^{-1} \textnormal{Cov}(\mathbf{x}_{t}, {z}_{t-\tau})^*  {\Gamma}_{\mathbf{x}}^{-1} \textnormal{Cov}(\mathbf{x}_{t}, {z}_{t-\tau})  \\
		& =
		\gamma(0)^{-1} \left( \sum_{k=1}^N  g_k(\tau) {c}_k \mathbf{{v}}_k \right)^\ast \left( \sum_{k,l=1}^N {c}_k \overline{c}_l \frac{h_{k,l}}{1 - \lambda_k \overline{\lambda}_l} \mathbf{{v}}_k \mathbf{{v}}_l^\ast \right)^{-1} \left(\sum_{k=1}^N  g_k(\tau) {c}_k \mathbf{{v}}_k \right)\\
		& =
		\gamma(0)^{-1} {\mathbf{c}}^\ast  {G(\tau)}^\ast \left( \textnormal{diag}\left( {\mathbf{c}} \right) \left( \frac{h_{k,l}}{1 - \lambda_k \overline{\lambda}_l} \right)_{k,l}^N \textnormal{diag}\left( {\mathbf{c}}^\ast \right) \right)^{-1} G(\tau)\, {\mathbf{c}}
		\\
		& =
		\boldsymbol{\iota}_N^\top \,  {G(\tau)}^\ast \left( \gamma(0) \left( \frac{h_{k,l}}{1 - \lambda_k \overline{\lambda}_l} \right)_{k,l}^N \right)^{-1} G(\tau) \, \boldsymbol{\iota}_N.
	\end{align*}
	Notice that the final expression does not depend on $\mathbf{C}$, which proves the neutrality of the memory capacity with respect to the choice of the input mask, as required. 
\end{proof}

\newpage
\section{Additional Plots}

\begin{figure}[h!]
	\centering
	\begin{subfigure}[b]{0.4\textwidth}
		\centering
		\includegraphics[width=\textwidth]{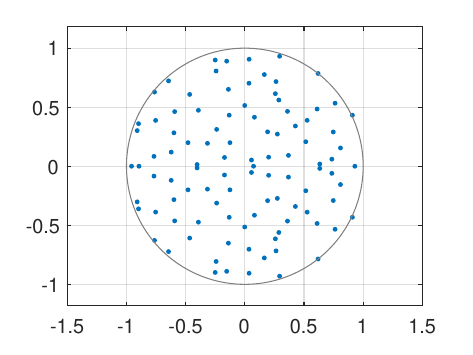}
		\caption{$A_{ij} \sim\ \text{i.i.d.}\ \mathcal{N}(0,1)$}
		\label{}
	\end{subfigure}
	%\hfill
	\begin{subfigure}[b]{0.4\textwidth}
		\centering
		\includegraphics[width=\textwidth]{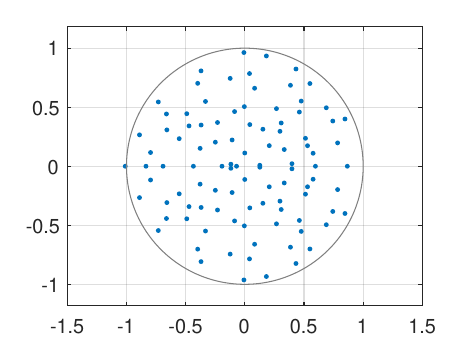}
		\caption{$A_{ij} \sim\ \text{i.i.d.}\ sp\mathcal{N}(0,1,0.1)$}
		\label{}
	\end{subfigure}
	\\[10pt]
	\begin{subfigure}[b]{0.4\textwidth}
		\centering
		\includegraphics[width=\textwidth]{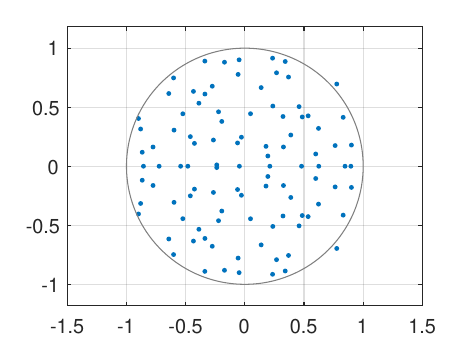}
		\caption{$A_{ij} \sim\ \text{i.i.d.}\ \mathcal{U}(-1,1)$}
		\label{}
	\end{subfigure}
	%\hfill
	\begin{subfigure}[b]{0.4\textwidth}
		\centering
		\includegraphics[width=\textwidth]{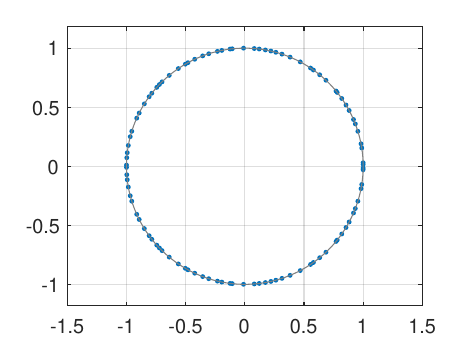}
		\caption{$A_{ij} \sim\ \mathcal{O}(\mathcal{N}(0,1))$}
		\label{}
	\end{subfigure}
	\\[10pt]
	%\hfill
	\begin{subfigure}[b]{0.4\textwidth}
		\centering
		\includegraphics[width=\textwidth]{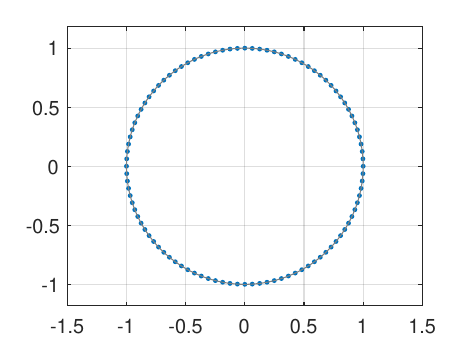}
		\caption{$A_{ij} \sim \text{cyclic}(N)$}
		\label{}
	\end{subfigure}
	\begin{subfigure}[b]{0.4\textwidth}
		\centering
		\includegraphics[width=\textwidth]{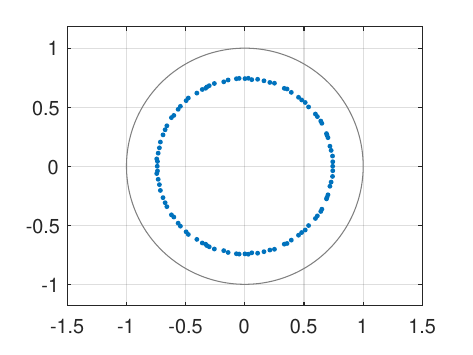}
		\caption{$A_{ij} \sim\ \text{i.i.d.}\ sp_C\mathcal{N}(0,1,0.1,0.7)$}
		\label{}
	\end{subfigure}
	\vspace{2em}
	\caption{Eigenvalues (blue) for random and non-random reservoir matrices and the complex unit circle (gray), $N = 100$. For specifications with entries $A_{ij} \sim \text{i.i.d.}\ \mathcal{N}$, $sp\mathcal{N}$ and $\mathcal{U}$ (upper row) the matrices are normalized according to the circular law rates ${N}^{-1/2}$, ${(0.1 N)}^{-1/2}$ and ${(N/3)}^{-1/2}$, respectively. In (f) $sp_C\mathcal{N}(0,1,0.1,0.7)$ stands for sparse standard Gaussian with sparsity degree $0.1$ and condition number $0.7$.}
	\label{fig:A_eigvals}
\end{figure}
\newpage
% ---------------------------------------------------------
% Bibliography

\vskip 0.2in

\bibliography{GOLibrary}
%\bibliography{RC23_bib}

% ----------------------------------------------------------

\end{document}